\documentclass[sigconf]{acmart}
\AtBeginDocument{%
  }

\setcopyright{acmlicensed}
\copyrightyear{2026}
\acmYear{2026}
\setcopyright{cc}
\setcctype{by}
\acmConference[MM '26]{Proceedings of the 34th ACM International Conference on Multimedia}{November 10--14, 2026}{Rio de Janeiro, Brazil}
\acmBooktitle{Proceedings of the 34th ACM International Conference on Multimedia (MM '26), November 10--14, 2026, Rio de Janeiro, Brazil}
\acmDOI{10.1145/3767308.3835730}
\acmISBN{979-8-4007-2213-4/2026/11}

\usepackage{amsmath}
\usepackage{mathtools}
\usepackage{amsthm}
\usepackage{xurl}
\usepackage{multirow}
\usepackage{bm}
\usepackage{mathtools}
\usepackage{mathrsfs}

\usepackage{tabularx}
\usepackage[table]{xcolor}
\definecolor{lightblue}{RGB}{220,235,255}
\definecolor{myblue}{RGB}{102,153,204}
\definecolor{myorange}{RGB}{230,180,80}
\definecolor{lightgraytext}{gray}{0.39}

\newtheorem{assumption}{Assumption}
\newcounter{theorem_count}

\usepackage{array}
\newcolumntype{s}{>{\hsize=0.7\hsize\raggedright\arraybackslash}X}
\newcolumntype{m}{>{\hsize=1.1\hsize\centering\arraybackslash}X}

\acmSubmissionID{4544}

\begin{document}

\title{Stable Attention Response for Reliable Precipitation Nowcasting}


\author{Penghui Wen}
\email{penghui.wen@sydney.edu.au}
\affiliation{%
  \institution{The University of Sydney}
  \country{Camperdown, NSW, Australia}
}

\author{Zexin Hu}
\email{zehu4485@uni.sydney.edu.au}
% \additionalaffiliation{%
%     \institution{Ant Group}
%     \country{Hangzhou, Zhejiang, China}
% }
\affiliation{%
  \institution{The University of Sydney}
  \country{Camperdown, NSW, Australia}
}
\affiliation{%
  \institution{Ant Group}
  \country{Hangzhou, Zhejiang, China}
}

\author{Sen Zhang}
\email{senzhang.thu10@gmail.com}
\affiliation{%
  \institution{Independent Researcher}
  \country{Sydney, NSW, Australia}
}

\author{Patrick Filippi}
\email{patrick.filippi@sydney.edu.au}
\affiliation{%
 \institution{The University of Sydney}
 \country{Camperdown, NSW, Australia}
 }

\author{Xiaogang Zhu}
\email{xiaogang.zhu@adelaide.edu.au}
\affiliation{%
    \institution{Adelaide University}
  \country{Adelaide, SA, Australia}}

\author{Allen Benter}
\email{allen.benter@dpird.nsw.gov.au}
\affiliation{%
 \institution{Orange Agricultural Institute}
 \country{Orange, NSW, Australia}
 }

\author{Thomas Bishop}
\email{thomas.bishop@sydney.edu.au}
\affiliation{%
 \institution{The University of Sydney}
 \country{Camperdown, NSW, Australia}
 }

\author{Zhiyong Wang}
\email{zhiyong.wang@sydney.edu.au}
\affiliation{%
 \institution{The University of Sydney}
 \country{Camperdown, NSW, Australia}
 }

\author{Kun Hu}
\email{k.hu@ecu.edu.au}
\authornotemark[1]
\affiliation{%
  \institution{Edith Cowan University}
  \country{Joodalup, WA, Australia}}

\renewcommand{\shortauthors}{Penghui Wen et al.}

\begin{abstract}

Precipitation nowcasting remains challenging due to the highly localized, rapidly evolving, and heterogeneous nature of atmospheric dynamics. Although recent methods adopt attention-based architectures in both unimodal and multimodal settings, they mainly emphasize stronger representation learning and prediction capacity, while paying less attention to the stability of attention responses across samples. In this work, we show that cross-sample instability of attention-response energy is an important and previously underexplored source of forecasting unreliability. Empirically, inaccurate forecasts are associated with larger attention-response energy variance across heads and layers. Theoretically, we show that cross-sample variability can propagate through self-attention, and enlarge a lower bound on prediction error.
Based on this insight, we propose {HARECast}, a {H}ead-wise {A}ttention {R}esponse {E}nergy-regulated framework for precipitation nowcasting. HARECast explicitly models head-wise attention-response energy and stabilizes it through a group-wise regularization objective that reduces cross-sample fluctuations. The proposed formulation is generic and applicable to both unimodal and multimodal nowcasting. We instantiate HARECast in a standard forecasting pipeline with reconstruction branches and a diffusion-based predictor, and evaluate it on commonly used benchmarks--SEVIR and MeteoNet. Experiments demonstrate that HARECast achieves state-of-the-art performance. Our code is available at \textit{https://github.com/ph-w2000/HARECast}.

\end{abstract}

\begin{CCSXML}
<ccs2012>
   <concept>
       <concept_id>10010405.10010432.10010437</concept_id>
       <concept_desc>Applied computing~Earth and atmospheric sciences</concept_desc>
       <concept_significance>500</concept_significance>
       </concept>
   <concept>
       <concept_id>10010147.10010178.10010224</concept_id>
       <concept_desc>Computing methodologies~Computer vision</concept_desc>
       <concept_significance>500</concept_significance>
       </concept>
 </ccs2012>
\end{CCSXML}

\ccsdesc[500]{Applied computing~Earth and atmospheric sciences}
\ccsdesc[500]{Computing methodologies~Computer vision}

\keywords{Precipitation nowcasting; Multimodal modeling; Energy regulartion}


\maketitle

\section{Introduction}
Precipitation nowcasting aims to predict near-future weather evolution from recent observation and has become increasingly essential in daily planning, transportation, and disaster management~\cite{zhang2023skilful}. 
Despite substantial progress in deep forecasting models, reliable nowcasting remains challenging because atmospheric dynamics are highly localized, rapidly evolving, and strongly heterogeneous across samples~\cite{wen2026duocast}. As a result, even small spatial deviations in predicted rainfall patterns can noticeably degrade forecast quality.

To address such complex spatio-temporal variability, recent nowcasting methods increasingly adopt attention-based~\cite{vaswani2017attention} architectures. This trend appears in both unimodal settings, such as radar-based forecasting~\cite{yu2024diffcast, fengperceptually}, and multimodal settings that combine radar with satellite observations~\cite{yu2025integrating, yu2025pimmnet}. Attention is attractive because it can adaptively focus on informative regions and capture long-range interactions that are difficult to model with purely convolutional designs. 
As a result, improving representational capacity, generative prediction or cross-modal fusion has become a major direction in modern nowcasting approaches.

However, we argue that a more fundamental issue has been underexplored: \emph{the stability of the attention response itself}. In nowcasting, samples often differ substantially in intensity, spatial coverage, displacement patterns, and cross-source consistency~\cite{yu2025integrating,zhang2023skilful}, as shown in Figure~\ref{fig:motivation} (b). Under such regime changes, the responses of attention heads may fluctuate sharply across samples. This instability already arises in unimodal forecasting and can become even more pronounced when heterogeneous modalities are introduced. As a result, attention-based models may exhibit unstable internal responses across samples, creating a previously overlooked bottleneck for reliable forecasting.

\begin{figure}[t]
  \centering
   \includegraphics[width=\linewidth]{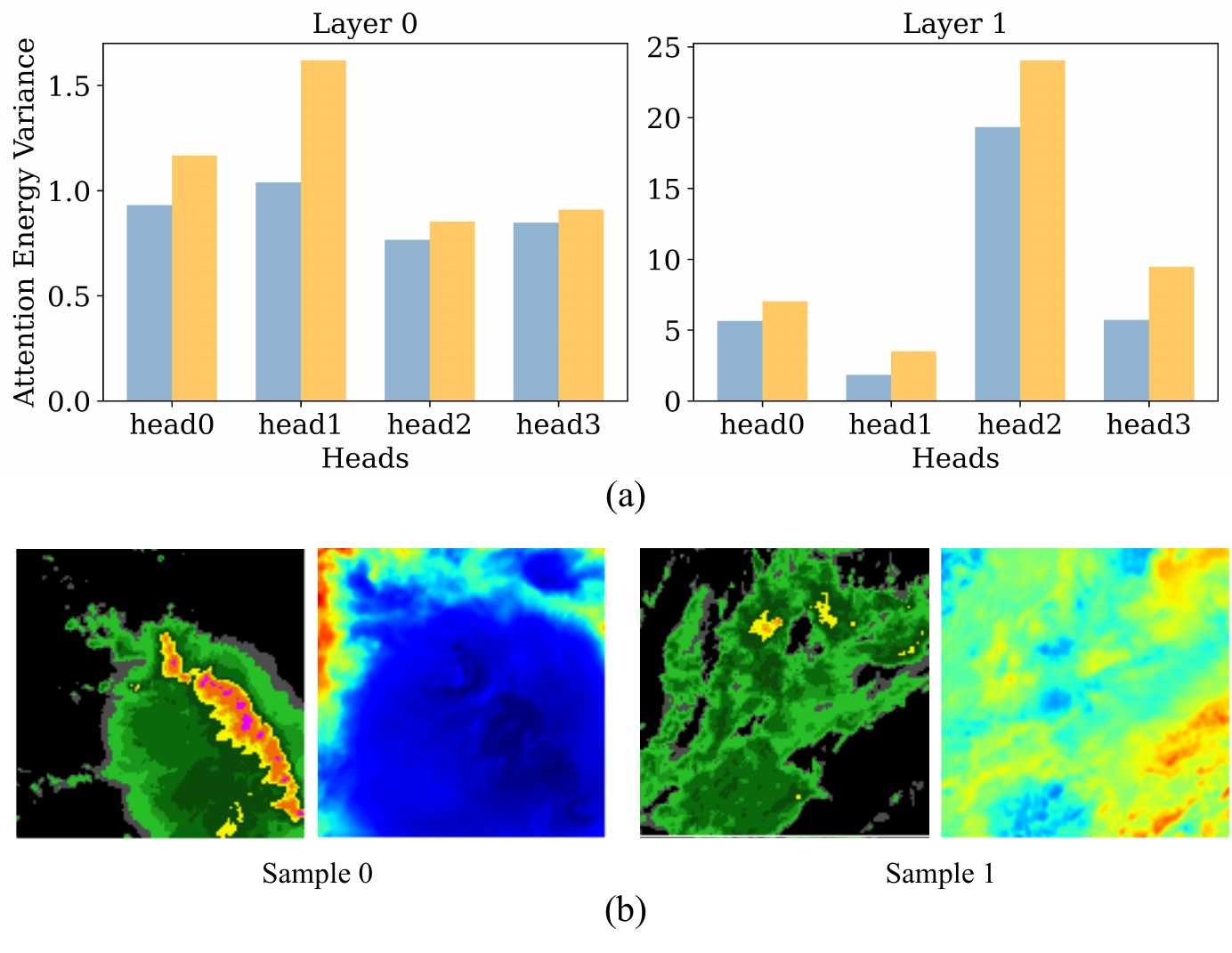}
   \caption{
(a) Visualization of the variance of attention-response energy across samples within a batch, shown for different heads and layers of PiMMNet~\cite{yu2025pimmnet}. An
\textcolor{myblue}{accurately predicted batch} exhibit consistently lower variance, while an 
\textcolor{myorange}{inaccurately predicted batch} shows higher variance. (b) Examples of precipitation patterns exhibiting rapid and large spatial variations, which lead to strong cross-sample differences and high attention-response variance.
   }
   \label{fig:motivation}
\end{figure}

We study this phenomenon from both empirical and theoretical perspectives.
Empirically, we observe a clear correlation between prediction quality and the cross-sample variance of attention-response energy: accurate batch tend to exhibit more stable head responses, while inaccurate batch are often associated with larger energy fluctuations across heads and layers, as shown in Figure~\ref{fig:motivation} (a). 
Theoretically, we show that cross-sample variability in the input can propagate through self-attention and induce variance in the attention response, which enlarges a lower bound on the prediction error.
These results suggest that expressive attention alone is insufficient for reliable nowcasting, the attention response should also remain \emph{stable across samples}.

Motivated by this observation, we propose \textbf{HARECast}, a \textbf{H}ead-wise \textbf{A}ttention \textbf{R}esponse \textbf{E}nergy-regulated diffusion framework for precipitation nowcasting.
The central idea of HARECast is to explicitly model and stabilize the energy of attention-induced responses.
Specifically, we define head-wise attention response energy (HARE) to quantify the effective response strength of each attention head, and use it to organize heads into groups with different response patterns.
We then introduce a group-wise HARE stabilization objective that regularizes attention-response energy toward the mean response of each head group,thereby reducing cross-sample fluctuations and stabilizing forecasting behavior. 
Importantly, this formulation is general: it can be applied to both unimodal and multimodal nowcasting network architectures, and is especially useful when heterogeneous modalities introduce additional response instability.

We instantiate HARECast within a standard nowcasting framework by combining the proposed HARE stabilization mechanism with modality-aware reconstruction branches and a diffusion-based predictor for future radar forecasting. This design allows the method to be integrated into common forecasting architectures while delivering state-of-the-art performance on benchmark datasets including SEVIR~\cite{veillette2020sevir} and MeteoNet~\cite{larvor2020meteo}.

Our main contributions are summarized as follows:
\begin{itemize}
\item We show that cross-sample instability of attention-response energy is an important and previously underexplored failure mode in precipitation nowcasting.
\item We provide empirical and theoretical evidence that cross-sample variability can propagate through self-attention and increase a lower bound on prediction error, motivating explicit stabilization of attention-response energy.
\item We propose \textbf{HARECast}, a head-wise attention response energy-regulated framework for nowcasting, and demonstrate state-of-the-art performance on benchmark datasets.
\end{itemize}

\section{Related Work}
We briefly review related work on unimodal and multimodal precipitation nowcasting, and attention energy analysis.

\noindent\textbf{Unimodal Precipitation Nowcasting.}
Deep precipitation nowcasting methods have evolved from deterministic spatiotemporal prediction models to probabilistic and hybrid generative frameworks. Early \textit{deterministic models} learn spatiotemporal dynamics for mean-centric precipitation~\cite{shi2015convolutional}: PhyDNet~\cite{guen2020disentangling} incorporates physics-guided disentanglement to separate precipitation dynamics from residual information; FACL~\cite{yan2024fourier} replaces L2 with a signal-oriented spatiotemporal loss; AlphaPre~\cite{Lin_2025_CVPR} factorizes phase and amplitude to decouple motion from intensity. PercpCast~\cite{fengperceptually} incorporates perceptual constraints through a rectified flow model to achieve better perceptual performance. 

More recent \textit{probabilistic models} introduce latent variables to represent the intrinsic uncertainty of future weather, enabling finer depiction of small-scale processes~\cite{ravuri2021skilful}: DGMR uses a GAN with separate spatial/temporal discriminators~\cite{ravuri2021skilful}. NowcastNet~\cite{zhang2023skilful} injects physical priors into a GAN. Prediff adds knowledge-guided sampling~\cite{gao2024prediff}. DuoCast~\cite{wen2026duocast} proposes a dual-diffusion framework for precipitation forecasting that decomposes predictions into low- and high-frequency components modeled in orthogonal latent subspaces. A theoretical analysis proves that this frequency decomposition reduces prediction error. \textit{Hybrid models} combine deterministic backbones for large-scale evolution with probabilistic modules for local variability, balancing accuracy and diversity: CasCast~\cite{gong2024cascast} and DiffCast~\cite{yu2024diffcast} exemplify this pairing of trajectory prediction with stochastic refinement. These methods have substantially improved forecasting quality, but their main focus remains on predictive capacity, uncertainty modeling, or generative realism, rather than on the cross-sample stability of internal attention responses.

% MAU~\cite{chang2021mau}
% ConvGRU~\cite{shi2017deep}
% STRPM~\cite{chang2022strpm}

\noindent\textbf{Multimodal Precipitation Nowcasting.}
To overcome the limitations of single-source observations, recent studies increasingly incorporate multiple meteorological modalities, especially radar and satellite data. Existing multimodal precipitation nowcasting methods mainly differ in how they fuse complementary information. Rain-F~\cite{choi2021rain} and MetNet-2~\cite{espeholt2022deep} adopts an early-fusion strategy to process multimodal data. It aligns heterogeneous inputs, including radar and satellite, and concatenates them into a unified tensor along the channel dimension. CrossViViT~\cite{boussif2023improving} introduces an intermediate fusion mechanism at the token level after initial feature extraction. STJointNet~\cite{zheng2024cross} adopts a multitask learning framework to jointly forecast multiple meteorological modalities, highlighting the advantages of cross-modal learning. IMLSP~\cite{yu2025integrating} introduces a temporally adaptive multimodal attention module to enhance cross-modal interactions in the latent space, and employs a flow-based distribution adaptor to refine predictions. PiMMNet~\cite{yu2025pimmnet} proposes a dedicated attention mechanism grounded in the advection–diffusion principle, enabling effective integration of radar and satellite observations for improved prediction. 

While these methods demonstrate the value of cross-source complementarity, they primarily aim to enhance multimodal representation learning and fusion quality. In contrast, our work focuses on a more general issue that arises in both unimodal and multimodal settings, namely the cross-sample stability of attention-response energy, which can become even more critical when heterogeneous inputs are involved.

\noindent\textbf{Attention Energy Analysis.}
A large body of work has studied the role of attention heads in Transformer models. Prior studies show that not all heads contribute equally~\cite{michel2019sixteen}, and many methods exploit it by adjusting and pruning heads to reduce computation and improve accuracy. The norm-based analysis~\cite{kobayashi2020attention} studies transformers by jointly examining attention weights and the norms of head outputs.  X-Pruner~\cite{yu2023x} proposes an explainability-aware masking scheme that estimates each prunable unit’s class-wise contribution for structured pruning. $\sigma$Reparam~\cite{zhai2023stabilizing} links low attention entropy to training instability, often observed as oscillatory loss. 
However, these studies mainly focus on model compression, interpretability, or training dynamics in general Transformer models. They do not investigate whether cross-sample fluctuations in attention-response energy constitute a reliability bottleneck for forecasting, nor do they provide a mechanism to explicitly stabilize such responses in precipitation nowcasting. In this work, we investigate, both empirically and theoretically, the relationship between prediction quality and the cross-sample variance of attention-response energy, and further propose a method to stabilize the energy of attention-induced responses for improved forecasting performance.

\section{Methodology}
\begin{figure*}[t]
  \centering
   \includegraphics[width=\linewidth]{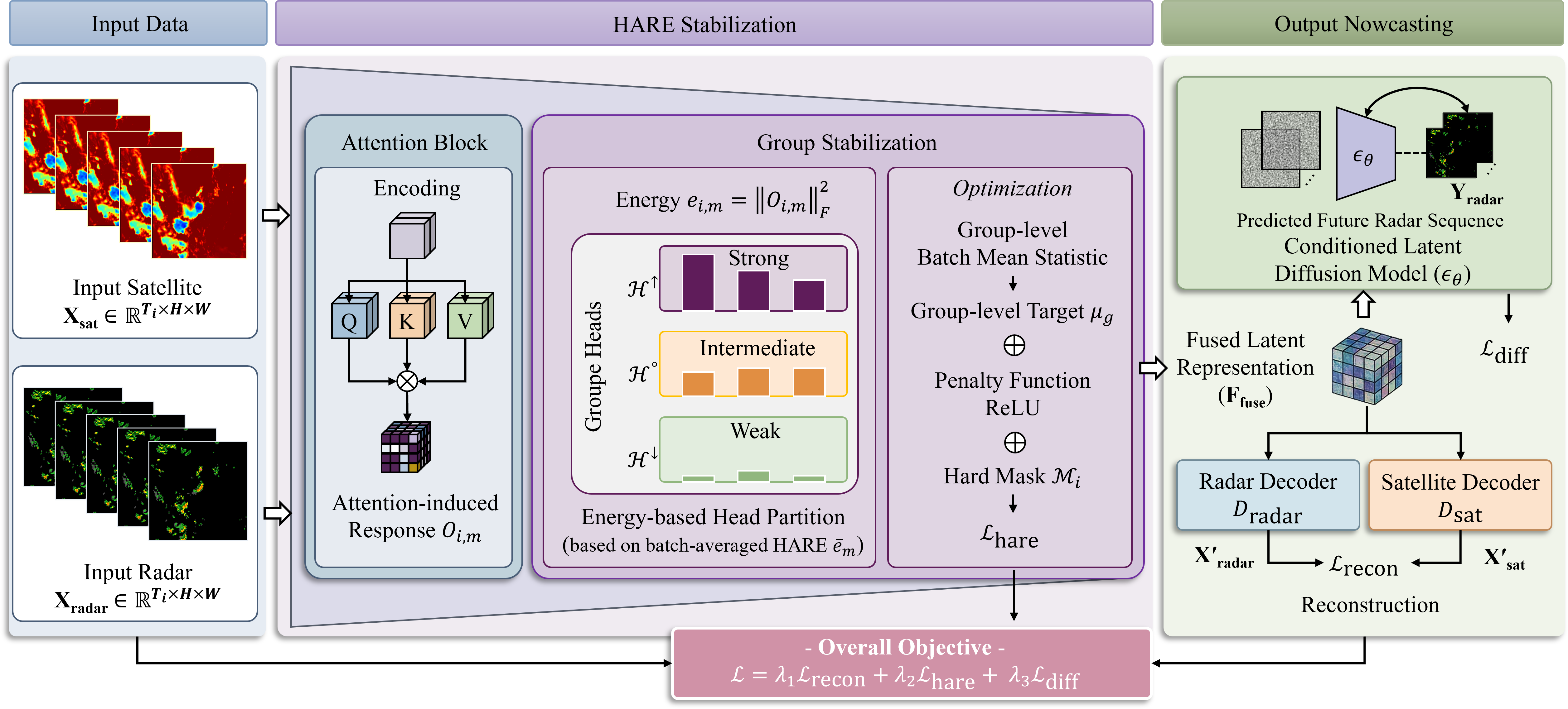}
  \caption{
Overview of \textbf{HARECast}. Given historical radar observations and, when available, satellite inputs, the model first encodes the input sequence with self-/cross-attention blocks. It computes head-wise attention energies, which are then used to partition heads into strong, intermediate, and weak groups according to batch energy statistics. The proposed HARE stabilization objective regularizes attention-response energy toward group-wise reference means, suppressing cross-sample fluctuations. The stabilized fused latent representation is subsequently used for conditioning a diffusion model for nowcasting. The full model is trained jointly with reconstruction, HARE stabilization, and diffusion generative objectives.
}
   \label{fig:architecture}
\end{figure*}

\subsection{Overview}
\noindent\textbf{Problem formulation.}
We study precipitation nowcasting from historical observations in both unimodal and multimodal settings. In the unimodal setting, given past radar observations $\mathbf{X}_{\text{radar}}\in\mathbb{R}^{T_I\times H\times W}$, the goal is to predict future radar frames $\mathbf{Y}_{\text{radar}}\in\mathbb{R}^{T_O\times H\times W}$ by learning $p(\mathbf{Y}_{\text{radar}}\mid \mathbf{X}_{\text{radar}})$. In the multimodal setting, the input additionally includes satellite observations $\mathbf{X}_{\text{sat}}\in\mathbb{R}^{T_i\times H\times W}$, and the goal becomes learning $p(\mathbf{Y}_{\text{radar}}\mid \mathbf{X}_{\text{radar}}, \mathbf{X}_{\text{sat}})$. Here, $H$ and $W$ denote the spatial dimensions, while $T_I$ and $T_O$ denote the lengths of the input and output sequences, respectively. To achieve longer forecast horizons, predictions are generated autoregressively.

\noindent\textbf{Core idea.}
Our method is built on one key observation: samples with strong localized convection, abrupt spatial shifts, and complex inter-modal discrepancies often induce highly unstable attention responses across samples (see Figure~\ref{fig:motivation}). This instability propagates to the prediction and degrades forecast reliability. We therefore design the method around \emph{stable attention energy across samples}. %, rather than only improving feature fusion or generation quality.

\subsection{Attention Variability \textit{vs.} Prediction Error}
\label{sec:energy_routing}

We provide a theoretical motivation for energy stabilization by analyzing how cross-sample variability propagates through self-attention and increases a lower bound on prediction error. This analysis justifies the need to explicitly stabilize attention energy across samples.
Let ${F(\mathbf{X})}$ be self-attention mapping obtained by \textit{queries} and \textit{keys}, where $\mathbf{X}$ is the input feature map.
For theoretical clarity, we consider a simplified predictor \footnote{For clarity, the analysis considers a self-attention block followed by a linear predictor. The same intuition extends to the full multi-head encoder and the diffusion-based predictor used in our framework.} (e.g., computations for \textit{values}) of the form
\begin{equation}
\hat{\mathbf{Y}} = \mathbf{G}(F(\mathbf{X})) = \mathbf{W}F(\mathbf{X}) + \mathbf{b},
\label{eq:linear_predictor}
\end{equation}
where $\mathbf{W}$ and $\mathbf{b}$ are the parameters of a linear prediction head.

\noindent\textbf{Assumption 1} \textit{(Variance non-degeneracy of self-attention).}
We assume that, on the data manifold, the self-attention mapping $\mathbf{F}$ does not collapse cross-sample variability. Specifically, there exists a constant $c_F > 1$ such that
\begin{equation}
\mathrm{Var}\!\bigl(F(\mathbf{X})\bigr)
\ge
c_F^2\,\mathrm{Var}(\mathbf{X}).
\label{eq:attn_var_assumption}
\end{equation}
Under Assumption 1, variability in the input necessarily induces non-trivial variability in the attention response.

\noindent\textbf{Lemma 1} \textit{(Variance propagation through the prediction head).}
Let $\hat{\mathbf{Y}} = \mathbf{W}F(\mathbf{X}) + \mathbf{b}$, and assume that the linear head is non-degenerate with minimum singular value
\begin{equation}
\sigma_{\min}(\mathbf{W}) \ge c_G > 0.
\label{eq:head_nondegenerate}
\end{equation}
Then the variance of the prediction is lower bounded by
\begin{equation}
\mathrm{Var}(\hat{\mathbf{Y}})
\ge
c_G^2\,\mathrm{Var}(F(\mathbf{X}))\ge
c_G^2 c_F^2\,\mathrm{Var}(\mathbf{X}).
\label{eq:linear_var_bound}
\end{equation}

\noindent\textbf{Lemma 2} \textit{(Lower bound on prediction error).}
Define the mean squared error (MSE) as
$
\mathrm{MSE}
=
\mathbb{E}\bigl\|\mathbf{Y}-\hat{\mathbf{Y}}\bigr\|_2^2.
\label{eq:mse_def}
$.
The following lower bound holds:
\begin{equation}
\mathrm{MSE}
\ge
\bigl(\mathbb{E}[\hat{\mathbf{Y}}] - \mathbb{E}[\mathbf{Y}]\bigr)^2
+
\left(
\mathrm{Var}(\hat{\mathbf{Y}})^{1/2}
-
\mathrm{Var}(\mathbf{Y})^{1/2}
\right)^2.
\label{eq:mse_lower_bound}
\end{equation}

As $\mathbf{X}$ and 
$\mathbf{Y}$ are drawn from temporally adjacent segments of the same stochastic precipitation process in this problem(with the reference time randomly sampled), it is natural to assume that their variances are comparable. We have the below major conclusion.

\refstepcounter{theorem_count}\noindent\textbf{Theorem 1}\label{thm:mse_lower_bound} \textit{(MSE lower bound induced by the attention response and the input variability).}
Assume that the linear prediction head is non-degenerate with minimum singular value
$\sigma_{\min}(W)\ge c_G>0$.
Assume further that the attention response satisfies
$\mathrm{Var}(F(\mathbf{X})) \ge c_F^2 \mathrm{Var}(\mathbf{X})$
for some constant $c_F>0$, and that
$
\mathrm{Var}(\mathbf{Y})=\mathrm{Var}(\mathbf{X}).
\label{eq:thm_matched}
$
If
$c_G c_F > 1$,
then the mean squared error satisfies
\begin{equation}
\mathrm{MSE}
\ge
\bigl( \mathbb{E}[\hat{\mathbf{Y}}] - \mathbb{E}[\mathbf{Y}] \bigr)^2
+
\left(c_G{\mathrm{Var}(F(\mathbf{X}))}^{\frac{1}{2}}-{\mathrm{Var}(\mathbf{Y})}^{\frac{1}{2}}\right)^2,
\label{eq:thm_bound_F}
\end{equation}
and consequently
\begin{equation}
\mathrm{MSE}
\ge
\bigl( \mathbb{E}[\hat{ \mathbf{Y}]} - \mathbb{E}[\mathbf{Y}] \bigr)^2
+
\left(c_G-\frac{1}{c_F}\right)^2 \mathrm{Var}(F\mathbf(\mathbf{X})),
\label{eq:thm_bound_F2}
\end{equation}
as well as
\begin{equation}
\mathrm{MSE}
\ge
\bigl(\mathbb{E}[\hat{\mathbf{Y}}]-\mathbb{E}[\mathbf{Y}] \bigr)^2
+
(c_G c_F - 1)^2 \mathrm{Var}(\mathbf{\mathbf{Y}}).
\label{eq:thm_bound_X}
\end{equation}
In particular, if the predictor is unbiased, i.e.,
$\mathbb{E}[\hat {\mathbf{Y}}]=\mathbb{E}[\mathbf{Y}]$,
then the last two bounds reduce to
\begin{equation}
\mathrm{MSE}\ge \left(c_G-\frac{1}{c_F}\right)^2 \mathrm{Var}(F(\mathbf{X})),
\mathrm{MSE}\ge (c_G c_F - 1)^2 \mathrm{Var}(\mathbf{X}).
\end{equation}

\noindent\textbf{Implication.}
Theorem~\ref{thm:mse_lower_bound} suggests that excessive variance in the attention response $F(\mathbf{X})$ is undesirable for reliable forecasting. When attention responses fluctuate too strongly across samples, the resulting variance can propagate through the prediction head, enlarge the discrepancy between prediction variance and target variance, and thereby increase a lower bound on the MSE. This provides the theoretical motivation for our energy stabilization design: instead of allowing attention-response energy to vary uncontrollably across samples, we explicitly regulate it to suppress harmful fluctuations while preserving meaningful regime-dependent differences. In this way, we aim to reduce variance propagation from the attention module to $\hat{\mathbf{Y}}$ and improve forecasting reliability. %Further details and proofs are available in the Supplementary.

%\noindent\textbf{Implication.}
%These results indicate that large variance for $F(\mathbf{X})$ in the attention module is undesirable for reliable forecasting: if attention responses fluctuate too strongly across samples, they can unnecessarily increase the variance of the final prediction and thus worsen the MSE lower bound. This motivates our energy stabilization design. Instead of allowing attention energies to vary uncontrollably, we explicitly regulate their variance, so that the attention module remains stable across meteorological samples while still preserving meaningful regime-dependent differences. By controlling the variance of the attention response, we aim to reduce variance propagation to $\hat{\mathbf{Y}}$ and thereby lower the resulting MSE.

\subsection{Group-wise HARE Stabilization}

We stabilize attention responses at the group level by explicitly regularizing head-wise attention response energy across samples. This design suppresses harmful cross-sample variability while preserving meaningful differences in response strength, thereby reducing the risk of error amplification.

\noindent\textbf{Head-wise Attention Response Energy (HARE).}
Consider a multi-head self-attention block with heads in $H_1,H_2,\dots,H_M$.
For the $i$-th sample in a batch and the $m$-th head, let $\mathbf{A}_{i,m}\in\mathbb{R}^{N\times N}$ denote the attention map obtained through the query and key tensors, and $\mathbf{V}_{i,m}\in\mathbb{R}^{N\times d_h}$ denote the value tensor, where $d_h$ is the dimensionality of each head. 
The attention-induced response is:
\begin{equation}
\mathbf{O}_{i,m} = \mathbf{A}_{i,m}\mathbf{V}_{i,m}.
\end{equation}

We define the energy of the $m$-th head \textit{w.r.t.} the $i$-th sample as:
\begin{equation}
e_{i,m} = \|\mathbf{O}_{i,m}\|_F^2.
\end{equation}
This quantity reflects the effective strength of the attention response that is actually delivered to downstream prediction.

The batch-averaged head energy for the $m$-th head is:
\begin{equation}
\bar e_m = \frac{1}{B}\sum_{i=1}^{B} e_{i,m},
\end{equation}
where $B$ is the batch size. %We further denote
%$
%{\mathcal{E}}=\{\bar e_m\}
%$
%denote the corresponding set of head energies.

\noindent\textbf{Energy-based Head Partition.}
Based on the batch-averaged HARE, we partition attention heads into three groups for group-wise stabilization: dominant heads with the strongest responses, inactive heads with weak responses, and contextual heads in between. Let
\begin{equation}
\bar e = \frac{1}{M}\sum_{m=1}^{M} \bar e_m
\label{eq:global_head_energy}
\end{equation}
be the mean energy across all heads. 

We isolate the strongest-response head, collect weakly activated heads, and treat the remainder as an intermediate group:
\begin{equation}
%\left\{
\begin{aligned}
\mathcal H^{\uparrow} = \{\arg\max_{m}\bar e_m\},
\mathcal H^{\downarrow} = \{m\mid \bar e_m<\alpha \bar e\},
\mathcal H^{\circ} &= [M]\setminus\bigl(\mathcal H^{\uparrow}\cup\mathcal H^{\downarrow}\bigr),
\end{aligned}
%\right.
\label{eq:head_groups}
\end{equation}
where $\alpha\in(0,1)$ is a fixed threshold. 

The group-level energy for the $i$-th sample is defined as
\begin{equation}
e_i^{g}
=
\frac{1}{|\mathcal H^{g}|}
\sum_{m\in\mathcal H^{g}} e_{i,m},
\qquad
g\in\{\uparrow,\circ,\downarrow\}.
\label{eq:group_energy}
\end{equation}
This grouping provides a simple and robust basis for group-wise energy stabilization.

\noindent\textbf{Stablization Loss.}
To stabilize attention energy without collapsing all samples to the same response level, we define the target energy for each head group $g\in\{\uparrow,\circ,\downarrow\}$ based on the batch mean statistic of that group:
\begin{equation}
\mu_g = \frac{1}{B}\sum_{i=1}^{B} e_i^{g},
\end{equation}
which provides a stable estimate of the typical response scale of group $g$ in the current batch.

To this end, we calibrate each sample toward the batch-level target:
\begin{equation}
\mathcal{L}_{\mathrm{hare}}
=
\frac{1}{B}\sum_{i=1}^{B}\mathcal{M}_i
\sum_{g\in\{\uparrow,\circ,\downarrow\}}
\rho(\mu_g - e_i^{g}),
\end{equation}
where $\rho(\cdot)$ is a robust penalty function that ReLU is adopted for this purpose, and $\mathcal{M}_i\in\{0,1\}$ is a mask for selecting difficult samples.

\subsection{HARECast: Energy-regulated Forecasting}
Importantly, the HARECast formulation is general: it can be applied to both uni- and multimodal nowcasting architectures. Both unimodal and multimodal HARECast follow the same overall pipeline: an encoder first produces a latent representation and stabilizes its cross-sample attention-response energy using the proposed group-wise HARE regularization; lightweight reconstruction decoders are then used to preserve sufficient input information in the latent space; finally, a conditional diffusion model predicts future radar frames conditioned on the stabilized latent representation.

\noindent\textbf{Unimodal HARECast.}
In the unimodal setting, the encoder takes historical radar observations as input and produces a latent representation:
\begin{equation}
\mathbf{F} = E(\mathbf{X}_{\mathrm{radar}}).
\label{eq:uni_latent}
\end{equation}
To preserve input information, we reconstruct the radar observations from the same latent representation:
\begin{equation}
\hat{\mathbf{X}}_{\mathrm{radar}} = D_{\mathrm{radar}}(\mathbf{F}),
\label{eq:uni_reconstruction}
\end{equation}
with reconstruction loss
\begin{equation}
\mathcal{L}_{\mathrm{recon}}^{\mathrm{uni}}
=
\mathbb{E}\!\left[
\|\hat{\mathbf{X}}_{\mathrm{radar}}-\mathbf{X}_{\mathrm{radar}}\|_2^2
\right].
\label{eq:uni_recon_loss}
\end{equation}
The stabilized latent representation $\mathbf{F}$ is then used to condition a diffusion model for future radar prediction.

\noindent\textbf{Multimodal HARECast.}
In the multimodal setting, the encoder takes historical radar and satellite observations as input and produces a fused latent representation
\begin{equation}
\mathbf{F}_{\mathrm{fuse}} = E(\mathbf{X}_{\mathrm{radar}}\oplus \mathbf{X}_{\mathrm{sat}}),
\label{eq:fused_latent}
\end{equation}
where $\oplus$ is a concatenation operator. 
To encourage $\mathbf{F}_{\mathrm{fuse}}$ to retain the shared atmospheric structure required by both modalities, we reconstruct the radar and satellite observations through two modality-specific decoders:
\begin{equation}
\begin{aligned}
\hat{\mathbf{X}}_{\mathrm{radar}} = D_{\mathrm{radar}}(\mathbf{F}_{\mathrm{fuse}}), \quad
\hat{\mathbf{X}}_{\mathrm{sat}} = D_{\mathrm{sat}}(\mathbf{F}_{\mathrm{fuse}}).
\label{eq:reconstruction}
\end{aligned}
\end{equation}
The corresponding reconstruction objective is
\begin{equation}
\mathcal{L}_{\mathrm{recon}}^{\mathrm{multi}}
=
\mathbb{E}\!\left[
\|\hat{\mathbf{X}}_{\mathrm{radar}}-\mathbf{X}_{\mathrm{radar}}\|_2^2
+
\|\hat{\mathbf{X}}_{\mathrm{sat}}-\mathbf{X}_{\mathrm{sat}}\|_2^2
\right].
\label{eq:recon_loss}
\end{equation}
The stabilized fused representation $\mathbf{F}_{\mathrm{fuse}}$ is then used to condition a diffusion model for future radar prediction.

%\noindent\textbf{Generative Nowcasting.} Let $\epsilon_{\theta}(\cdot)$ denote the denoising network parameterized by $\theta$. Following a standard latent diffusion formulation, the model predicts future radar frames conditioned on the stabilized latent representation, and the corresponding forecasting loss in diffusion is denoted by $\mathcal{L}_{\mathrm{diff}}$.

\noindent\textbf{Generative Nowcasting.}
Let $\epsilon_{\theta}(\cdot)$ denote the denoising network parameterized by $\theta$. Following a standard latent diffusion formulation, the model predicts future radar frames conditioned on the stabilized latent representation, and the corresponding forecasting loss is denoted by $\mathcal{L}_{\mathrm{diff}}$. 

\noindent\textbf{Optimization.} 
The overall training objective is:
\begin{equation}
\mathcal{L}
=
\lambda_{1}\mathcal{L}_{\mathrm{recon}}^{(*)}
+
\lambda_{2}\mathcal{L}_{\mathrm{hare}}
+
\lambda_{3}\mathcal{L}_{\mathrm{diff}},
\label{eq:final_loss}
\end{equation}
where
$\mathcal{L}_{\mathrm{recon}}^{(*)}\in\{\mathcal{L}_{\mathrm{recon}}^{\mathrm{uni}},\mathcal{L}_{\mathrm{recon}}^{\mathrm{multi}}\}$
denotes the reconstruction loss used in the unimodal or multimodal setting, respectively. $\lambda_{1}$, $\lambda_{2}$, and $\lambda_{3}$ are hyperparameter to balance the three terms..

Overall, the framework follows a simple design principle: group-wise HARE stabilization controls cross-sample variability in attention responses, auxiliary reconstruction encourages the latent features to preserve informative structure, and latent diffusion converts the stabilized representation into final precipitation forecasting.

\section{Experiments \& Discussions}

\subsection{Experimental Settings}
\noindent\textbf{Datasets.}
For the unimodal setting, we follow the common practices~\cite{Lin_2025_CVPR} with two benchmarks, {SEVIR}~\cite{veillette2020sevir} and {MeteoNet}~\cite{larvor2020meteo}, which provides radar observations for precipitation events. We use 5 historical frames (25 minutes) to predict the next 20 frames (100 minutes). 
For the multimodal setting, we follow PiMMNet~\cite{yu2025pimmnet}, using 12 historical frames (60 minutes) to forecast the subsequent 36 frames (180 minutes) on SEVIR with radar and satellite inputs. The data distribution for the two settings are shown in Table~\ref{tab:data_distribution}. %Additional dataset details are available in the Supplementary.

\begin{table}[htbp]
\centering
\setlength{\tabcolsep}{6pt}
\caption{Dataset statistics for unimodal and multimodal settings. $N_{t r}$, $N_{v a}$ and $N_{t e}$ represent the number of samples in the training, validation, and test partitions, respectively.}
{%
\begin{tabular}{|l|cccccc|}
\hline
\multirow{2}{*}{Dataset} & \multicolumn{6}{c|}{Unimodal (100 mins precipitation evolution)} \\
\cline{2-7}
 & $N_{t r}$ & $N_{v a}$ & $N_{t e}$ & $(C, H, W)$ & $T_I$ & $T_O$  \\ 
\hline
SEVIR & 23,808 & 6,016 & 8,100 & $(1, 128,128)$ & 5 & 20 \\
MeteoNet & 6,308 & 1,310 & 1,310 & $(1,128,128)$ & 5 & 20   \\
\hline

\addlinespace[0.5em]

\hline
\multirow{2}{*}{Dataset} & \multicolumn{6}{c|}{Multimodal (180 mins precipitation evolution)} \\
\cline{2-7}
 & $N_{t r}$ & $N_{v a}$ & $N_{t e}$ & $(C, H, W)$ & $T_I$ & $T_O$  \\ 
\hline
SEVIR & 12,254 & 4,900 & 6,290 & $(2, 128,128)$ & 12 & 36 \\
\hline

\end{tabular}%
}
\label{tab:data_distribution}
\end{table}

\noindent\textbf{Evaluation Metrics.}
Following prior work~\cite{Lin_2025_CVPR}, we evaluate nowcasting using the mean Critical Success Index (CSI), Heidke Skill Score (HSS), Learned Perceptual Image Patch Similarity (LPIPS) and Structural Similarity Index (SSIM). CSI quantifies pixel-wise overlap after binarizing predictions and ground truth, while HSS measures skill relative to random chance. We report CSI-M (CSI averaged across thresholds) and CSI over the two highest thresholds to highlight performance on high-intensity precipitation, as defined by the datasets. SSIM and LPIPS assess perceptual quality and similarity. For the multimodal setting, we follow PiMMNet~\cite{yu2025pimmnet} and compute HSS and CSI at multiple pooling scales (4$\times$4 and 16$\times$16) to assess neighborhood aggregation. %We also report LPIPS for visual quality.

\noindent\textbf{Implementation Details.}
We train all models using AdamW with a learning rate of $1\times10^{-4}$. Following DDIM~\cite{song2020denoising}, we use 1{,}000 diffusion steps during training and 5 sampling steps at inference. We do not extensively tune the loss weights, and simply set $\lambda_1=\lambda_2=1$ and $\lambda_3=5$ in both unimodal and multimodal settings.
The reconstruction decoders $D_{\text{radar}}$ and $D_{\text{sat}}$ are used only during training and are disabled at inference time, introducing no additional inference overhead. All training are conducted on 2x NVIDIA H100 GPU. %More details are available in the Supplementary.

%\noindent\textbf{Implementation Details.}
%We use the AdamW optimizer with a learning rate of $1\times10^{-4}$ during training. For diffusion, we follow DDIM~\cite{song2020denoising}, using 1,000 diffusion steps during training and 5 steps for sampling. We do not perform extensive hyperparameter tuning for the loss weights. For both singe- and multimodal settings, we set $\lambda_1 = \lambda_2 = 1$ and $\lambda_3 = 5$. For unimodality baselines, we follow the AlphaPre~\cite{Lin_2025_CVPR} using radar data from the same precipitation events. For multimodal comparisons, we evaluate on the same events and report the corresponding SoTA results from PiMMNet~\cite{yu2025pimmnet}. $D_{\text{radar}}$ and $D_{\text{sat}}$ are lightweight linear layers used only during training and are disabled at inference time, incurring no additional computational overhead.
%All experiments run on NVIDIA H100 GPUs. More details are provided in Supplementary Materials.

\begin{table*}[htbp]
\setlength{\tabcolsep}{3.3pt}
\centering
\caption{Quantitative performance on SEVIR and MeteoNet for unimodal 100-minute precipitation nowcasting.}
{
\begin{tabular}{|l|cccccc|cccccc|}
\hline
\multirow{2}{*}{Method} & \multicolumn{6}{c|}{\cellcolor{gray!20}SEVIR (Unimodal)} & \multicolumn{6}{c|}{\cellcolor{gray!20}MeteoNet (Unimodal)} \\
\cline{2-13}
 & \textuparrow CSI-M & \textuparrow CSI-181 & \textuparrow CSI-219 & \textuparrow HSS & \textdownarrow LPIPS & \textuparrow SSIM & \textuparrow CSI-M & \textuparrow CSI-24 & \textuparrow CSI-32 & \textuparrow HSS & \textdownarrow LPIPS & \textuparrow SSIM \\
\hline
ConvGRU(2017)~\cite{shi2017deep} & 0.2903& 0.0879& 0.0350& 0.3619& 0.2654 &0.6100 &0.3401 &0.2990 &0.1431 &0.4667 &0.2544 &0.7833 \\
MVCD (2021)~\cite{voleti2022mcvd} & 0.2148&  0.0510 & 0.0306 & 0.2743 & 0.2170 & 0.5265& 0.2336&  0.2614&  0.1020 & 0.3393&  0.1652&  0.5414 \\
MAU(2021)~\cite{chang2021mau} & 0.3076&0.1071&0.0516&0.3863&0.3865&0.6505 & 0.3233&0.2839&0.0997&0.4452&0.3023&0.7897 \\
SimVP(2022)~\cite{gao2022simvp}  & 0.3108&0.1106&0.0517&0.3924&0.3894&0.6508&0.3351&0.3002&0.1130&0.4573&0.3410&0.7804 \\
STRPM (2022)~\cite{chang2022strpm}& 0.2512  &0.0643 & 0.0308 & 0.3277  &0.2577  &0.6513 & 0.2606  &0.2338  &0.0882 & 0.3688 & 0.2004 & 0.5996 \\
FourCastNet(2022)~\cite{pathak2022fourcastnet} &0.2686&0.0717&0.0339&0.3355&0.4216&0.5976&0.3027&0.2533&0.1085&0.4216&0.4654&0.6450\\
Earthformer(2022)~\cite{gao2022earthformer}  & 0.2892&0.0844&0.0245&0.3665&0.3921&0.6633 &0.3205&0.2884&0.1237&0.4491&0.3626&0.7772\\
PhyDNet(2020)~\cite{guen2020disentangling} &0.3017&0.1040&0.0278&0.3812&0.3697&0.6532 &0.3384&0.3194&0.1366&0.4673&0.2844&0.7823\\
EarthFarseer(2024)~\cite{wu2024earthfarsser} & 0.3004&0.0992&0.0413&0.3829&0.3825&0.6327&0.3404&0.3170&0.1372&0.4726&0.3001&0.7542\\
NowcastNet(2023)~\cite{zhang2023skilful}& 0.2791&0.0770&0.0351&0.3512&0.4041&0.6839&0.3427&0.3206&0.1598&0.4751&0.2877&0.7879  \\
PreDiff(2024)~\cite{gao2024prediff} &0.2659&0.0652&0.0237&0.3445&0.2644&0.5857&0.2969&0.2675&0.1210&0.4485&0.1431&0.7057\\
DiffCast(2024)~\cite{yu2024diffcast} &0.3050&0.1300&0.0582&0.3996 & \underline{0.2090}&0.6482 &0.3512&0.3340&0.1808&0.4846&0.1542&0.7887\\
CasCast(2024)~\cite{gong2024cascast}  & 0.2847&0.0950&0.0311&0.3559&0.2270&0.5682 & 0.3299 & 0.2982 & 0.1354 & 0.4651 & 0.1456 & 0.7553\\
FACL(2024)~\cite{yan2024fourier}  & 0.3127 & 0.1333 & 0.0636 & 0.4015 & 0.3252& 0.5660  & 0.3601 & 0.3311 & 0.1806 & 0.5013 & 0.2159  & 0.7759\\
AlphaPre(2025)~\cite{Lin_2025_CVPR}  &0.3259&0.1332&0.0545&0.4110&0.2855&\underline{0.6884}&0.3824&0.3633&0.2002&0.5164&0.1991&0.7968\\
PercpCast(2025)~\cite{fengperceptually} & 0.3201 & 0.1356 & 0.0592 & 0.4026& \textbf{0.1960} & 0.6649 & 0.3715 & 0.3641 & 0.1959 & 0.5046 & \textbf{0.1333} & 0.7881\\

DuoCast(2026)~\cite{wen2026duocast} & \underline{0.3375} & \underline{0.1818} & \textbf{0.1074} & \underline{0.4318} & 0.2263 &  0.6827 & \underline{0.3892} & \underline{0.3841} & \underline{0.2381} & \underline{0.5297} & 0.1720 & \underline{0.7981}\\

\hline

\rowcolor{lightblue}
Ours (Unimodal)  & \textbf{0.3443} & \textbf{0.1832} & \underline{0.0978} & \textbf{0.4369} & 0.2110 & \textbf{0.6922} & \textbf{0.3933}& \textbf{0.3885} & \textbf{0.2392} & \textbf{0.5301}& \underline{0.1415} & \textbf{0.8019}\\

\hline

\end{tabular}%
}
\label{tab:sota}
\end{table*}

\begin{table}[htbp]
\centering
\setlength{\tabcolsep}{0.5pt}
\caption{Quantitative performance on SEVIR for multimodal 180-minute precipitation nowcasting.}
{%
\begin{tabular}{|l|cccccc|}
\hline
\multirow{2}{*}{Method} & \multicolumn{6}{c|}{\cellcolor{gray!20}SEVIR (Multimodal)} \\
\cline{2-7}
 & \textuparrow CSI-M & \textuparrow CSI$_4$ & \textuparrow CSI$_{16}$ & \textuparrow CSI-219 & \textuparrow HSS  & \textdownarrow LPIPS  \\ 
\hline
Rain-F(2021)~\cite{choi2021rain}  & 0.201 & 0.210& 0.223& 0.011 &0.248& 0.415\\
STJointNet(2024)~\cite{zheng2024cross}& 0.213& 0.229& 0.253 & 0.011 & 0.266& 0.386  \\
PiMMNet(2025)~\cite{yu2025pimmnet} & \underline{0.238}& \underline{0.284}& \underline{0.395}& \underline{0.037} &\underline{ 0.308}&  \textbf{0.243} \\
\hline
\rowcolor{lightblue}
Ours (Multimodal) & \textbf{0.254}& \textbf{0.316}& \textbf{0.421}& \textbf{0.053} & \textbf{0.315}& \underline{0.246}  \\

\hline
\end{tabular}%
}
\label{tab:multimodal-sota}
\end{table}

\subsection{Overall Performance}
\noindent\textbf{Quantitative Analysis.}
Tables~\ref{tab:sota} and~\ref{tab:multimodal-sota} compare our method with recent state-of-the-art unimodal and multimodal nowcasting models, respectively. Overall, our method delivers consistent improvements across  datasets and unimodal/multimodal settings.

\begin{figure}[t]
  \centering
   \includegraphics[width=\linewidth]{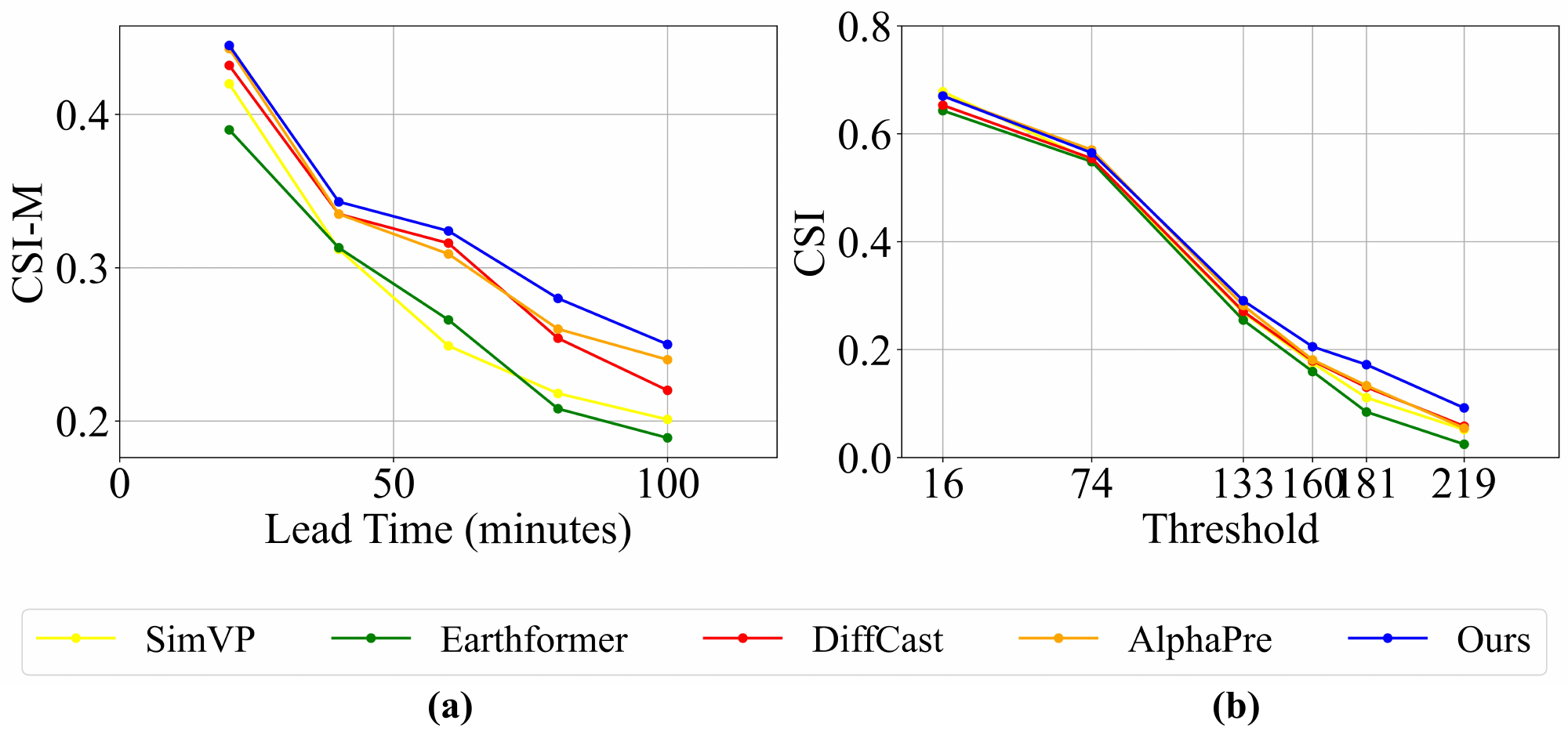}
\caption{
Forecasting performance on SEVIR (Unimodal) under varying forecast horizons and precipitation intensity thresholds. (a) Performance across forecast lead times. (b) Performance across precipitation intensity thresholds.}
   \label{fig:time_threshold}
\end{figure}

Compared with existing unimodal methods, our approach improves CSI-M by 2.0-2.4\% and HSS by 0.1-1.2\%, while remaining competitive on perceptual and structural metrics such as LPIPS and SSIM. The gains is pronounced under heavier rainfall conditions: at higher intensity thresholds, our method achieves approximately 0.8-1.1\% relative improvements on SEVIR and MeteoNet in CSI, suggesting a stronger ability to capture intense and localized precipitation patterns. 
Figure~\ref{fig:time_threshold} further illustrates this trend. In the subplot (a), which reports CSI over different prediction lead times, our model consistently outperforms recent baselines at nearly all horizons. In the subplot (b), which reports CSI under different intensity thresholds, our method achieves stronger performance overall, with especially clear advantages at higher thresholds. This shows that the proposed method is particularly effective in challenging forecasting regimes.

Compared with existing multimodal methods, our approach achieves a 7\% improvement in CSI-M and a 2\% gain in HSS, again demonstrating clear overall advantages. The improvements are especially notable on $\text{CSI}_{4}\text{-M}$ and $\text{CSI}_{16}\text{-M}$, indicating better performance on stronger precipitation events. This is consistent with our motivation that stabilizing attention-response energy is particularly beneficial when forecasting localized and high-intensity rainfall.

\noindent\textbf{Qualitative Analysis.}
\begin{figure}[t]
  \centering
   \includegraphics[width=\linewidth]{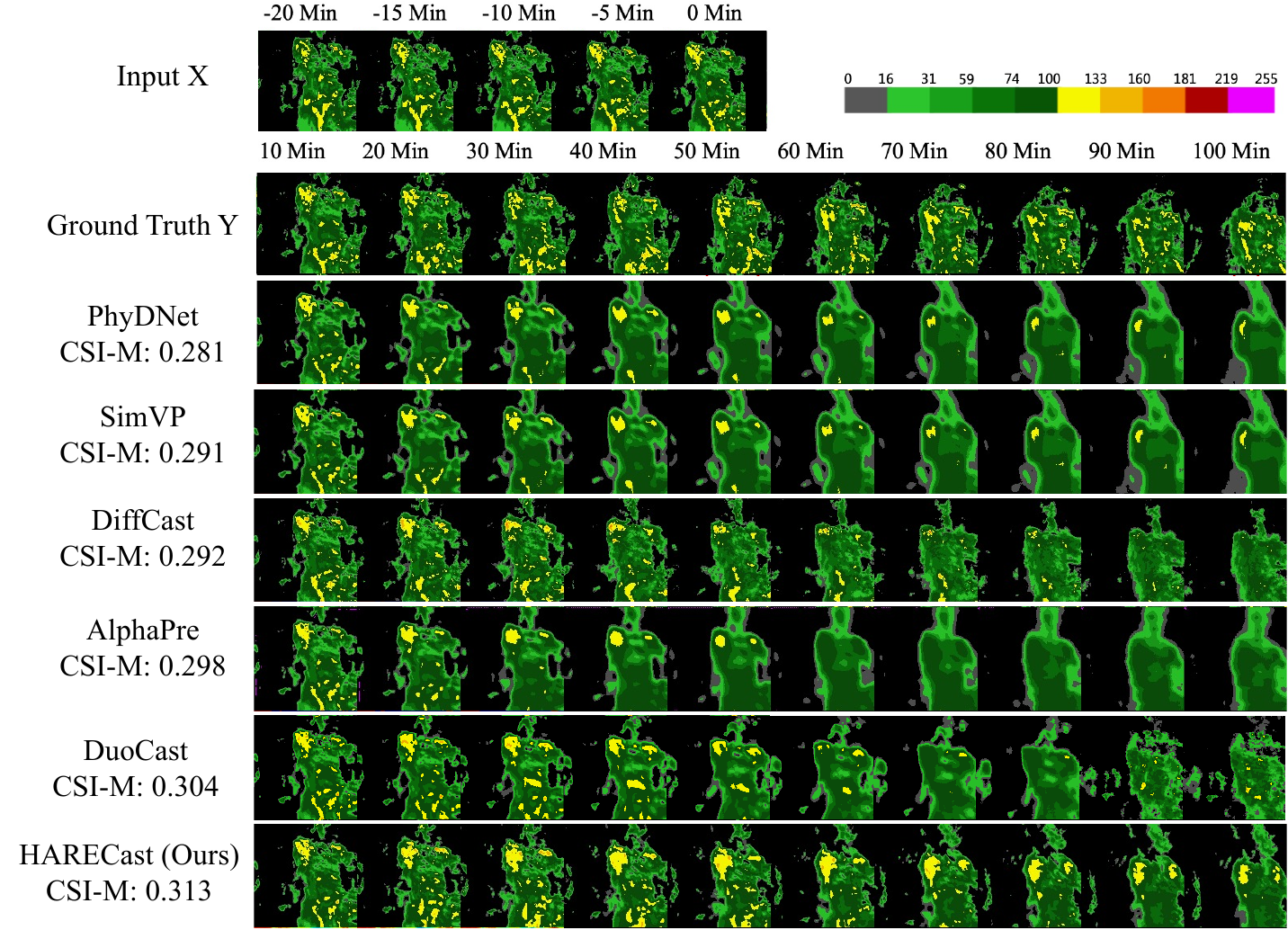}
   \caption{
Qualitative comparison on a SEVIR event (unimodal). HARECast produces forecasts with finer precipitation structures and more coherent temporal evolution than state-of-the-art methods. %\textcolor{lightgraytext}{Zoom in to see more details.}
   }
   \label{fig:qualitative}
\end{figure}
Figure~\ref{fig:qualitative} and Figure~\ref{fig:qualitative2} show qualitative comparisons with unimodal state-of-the-art methods on representative precipitation events from SEVIR and MeteoNet, respectively. Across both datasets, existing methods struggle to simultaneously preserve coherent large-scale evolution and fine-scale precipitation structure over long forecasting horizons. SimVP~\cite{gao2022simvp} and PhyDNet~\cite{guen2020disentangling} recover the precipitation pattern reasonably well at short lead times, but their predictions become progressively blurrier and lose local detail as the horizon increases. DiffCast~\cite{yu2024diffcast} and DuoCast~\cite{wen2026duocast} better retain fine-scale structures, yet still show noticeable deviations in long-range evolution.
For example, DiffCast produces sharper local structures, but tends to deviate from the ground-truth evolution, such as weakening the central medium-intensity precipitation region on SEVIR and consequently underestimating its intensity. In contrast, HARECast consistently preserves both the overall precipitation dynamics and local structural details, especially along medium-intensity boundaries and edge contours.

\begin{table}[htbp]
\centering
\renewcommand{\arraystretch}{1}
\setlength{\tabcolsep}{1.6pt}
\caption{Ablation studies of HARE stabilization (upper block) and group partition (lower block), where “w/o” denotes removing the module.}
{%
\begin{tabular}{|l|cccccc|}
\hline
\multirow{2}{*}{Method} & \multicolumn{6}{c|}{\cellcolor{gray!20}SEVIR (Unimodal)} \\
\cline{2-7}
 & \textuparrow CSI-M & \textuparrow CSI-181 & \textuparrow CSI-219 & \textuparrow HSS & \textdownarrow LPIPS  & \textuparrow SSIM  \\ 
\hline
HARECast & \textbf{0.344} & \textbf{0.183} & \textbf{0.098} & \textbf{0.437} & \textbf{0.211} & \textbf{0.692}\\
w/o G-HARE & 0.315 &  0.132 & 0.051 & 0.408 & 0.255 & 0.667 \\
\hline

HARECast & \textbf{0.344} & \textbf{0.183} & \textbf{0.098} & \textbf{0.437} & \textbf{0.211} & \textbf{0.692}\\
w/o Grouping & 0.318 & 0.155 & 0.070 & 0.413 & 0.232 & 0.652 \\

\hline
\end{tabular}%
}
\label{tab:ablation_1}
\end{table}

Figure~\ref{fig:qualitative3} presents qualitative comparisons with multimodal state-of-the-art methods on a representative SEVIR event. Across the full 30--180 minute forecasting horizon, HARECast produces forecasts that remain consistently closer to the radar ground truth. At early lead times, both STJointNet~\cite{zheng2024cross} and PiMMNet~\cite{yu2025pimmnet} capture the overall rainband geometry reasonably well. However, STJointNet quickly becomes overly smooth, leading to expanded precipitation regions and blurred internal structures, while PiMMNet preserves more local texture but introduces fragmented noisy patches and less coherent boundaries. By comparison, HARECast more faithfully maintains both the large-scale orientation of the rainband and the fine-scale high-intensity core, particularly the elongated yellow--red precipitation spine visible from 30 to 120 minutes. As forecasting becomes more difficult at longer lead times, HARECast continues to preserve a more realistic shrinking and displacement trend, together with cleaner boundaries and a more accurate intensity distribution. %, whereas the baselines tend either to over-smooth the precipitation field or to fragment the high-intensity regions.

%Figure~\ref{fig:qualitative} and Figure~\ref{fig:qualitative2} present qualitative comparisons with several unimodal state-of-the-art models on representative precipitation events from SEVIR and MeteoNet, respectively.   Across both datasets, existing methods generally struggle to jointly preserve the correct large-scale evolution and fine-scale precipitation details. SimVP~\cite{gao2022simvp} and PhyDNet~\cite{guen2020disentangling} capture the coarse precipitation pattern reasonably well at early lead times, but their predictions become increasingly blurry and lose local detail over time. AlphaPre~\cite{Lin_2025_CVPR}, DiffCast\cite{yu2024diffcast} and DuoCast~\cite{wen2026duocast} better preserve fine structures, yet often exhibit noticeable errors in long-range evolution; for example, on SEVIR, AlphaPre predicts a weakening band of medium-intensity precipitation (100–133) at 60 minutes. DiffCast recovers sharper local structures, but its forecasts tend to drift from the correct trend, such as weakening the central medium-intensity precipitation region on SEVIR and thus causing underestimation. In contrast, HARECast consistently captures both the large-scale precipitation dynamics and fine-scale local structures, with especially faithful preservation of medium-intensity edges and boundary contours.

\begin{figure}[t]
  \centering
   \includegraphics[width=\linewidth]{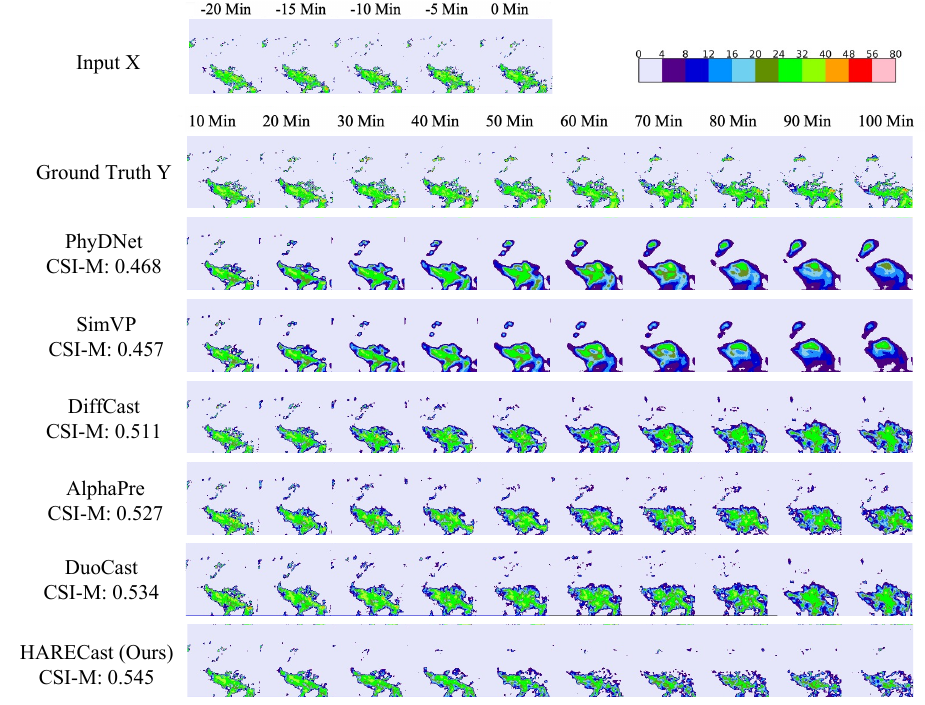}
   \caption{
   %Qualitative comparison on a MeteoNet event. HARECast recovers finer precipitation boundary structure and maintains more local details than SoTA baselines.
   Qualitative comparison on a MeteoNet event (unimodal). HARECast preserves finer precipitation boundaries and more local details than state-of-the-art methods. \textcolor{lightgraytext}{Zoom in to see more details.}
   }
   \label{fig:qualitative2}
\end{figure}

\begin{figure}[t]
  \centering
   \includegraphics[width=1\linewidth]{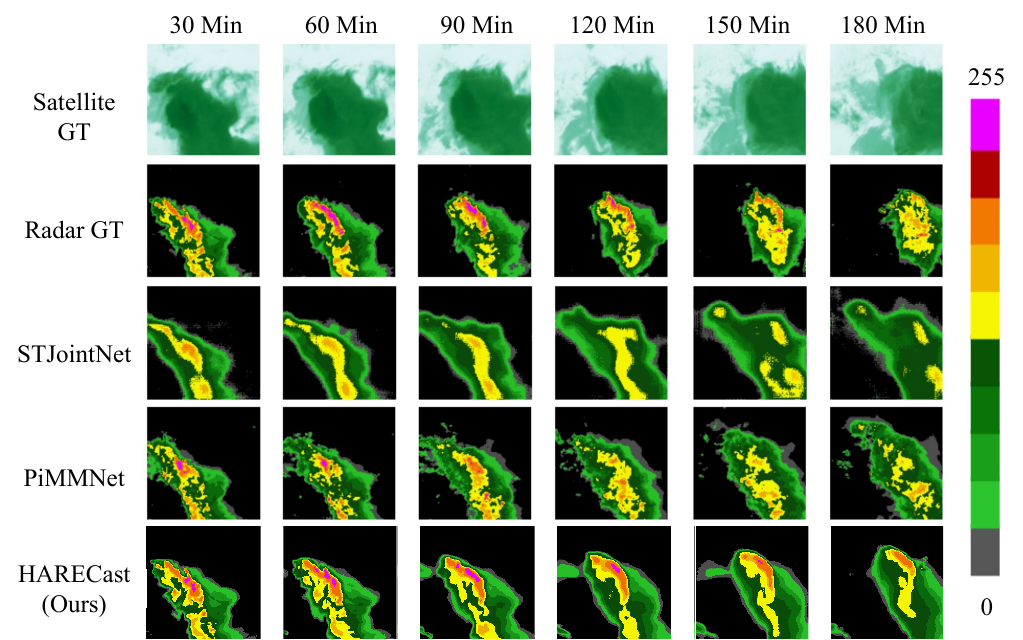}
   \caption{
   %Qualitative comparison on a SEVIR multimodal event. HARECast better preserves the large-scale rainband evolution the elongated high-intensity core and  sharper boundary structures than SoTA baselines.
   Qualitative comparison on a SEVIR event (multimodal). Compared with state-of-the-art baselines, HARECast better preserves the large-scale evolution of the rainband, the elongated high-intensity core, and fine boundary details. %\textcolor{lightgraytext}{Zoom in to see more details.}
   }
   \label{fig:qualitative3}
\end{figure}

\subsection{Ablation Study}

\noindent\textbf{Effectiveness of Group-wise HARE Stabilization.}
%As shown in Table~\ref{tab:ablation_1}, We compare the full model with versions excluding the group-wise HARE stabilization. Excluding the stabilization results in a significant drop in CSI, HSS, and SSIM, particularly at higher intensity thresholds, underscoring its contribution to forecasting accuracy. Figure~\ref{fig:ablation} visualizes attention energies across different heads with and without the energy stabilization mechanism. With energy stabilization, attention head energies are better aligned across samples, confirming that the mechanism reduces variability caused by cross-sample differences and supports more accurate and stable precipitation predictions.
Table~\ref{tab:ablation_1} upper block shows that removing HARE stabilization degrades all evaluation metrics. The drop is especially pronounced on higher threshold-based CSI metrics, indicating that the proposed stabilization is particularly beneficial for sharp and intense precipitation patterns. Figure~\ref{fig:ablation2} further shows that HARE stabilization makes head-wise attention-response energies more consistent across samples, supporting its role in suppressing cross-sample fluctuations.

\begin{figure}[t]
  \centering
   \includegraphics[width=\linewidth]{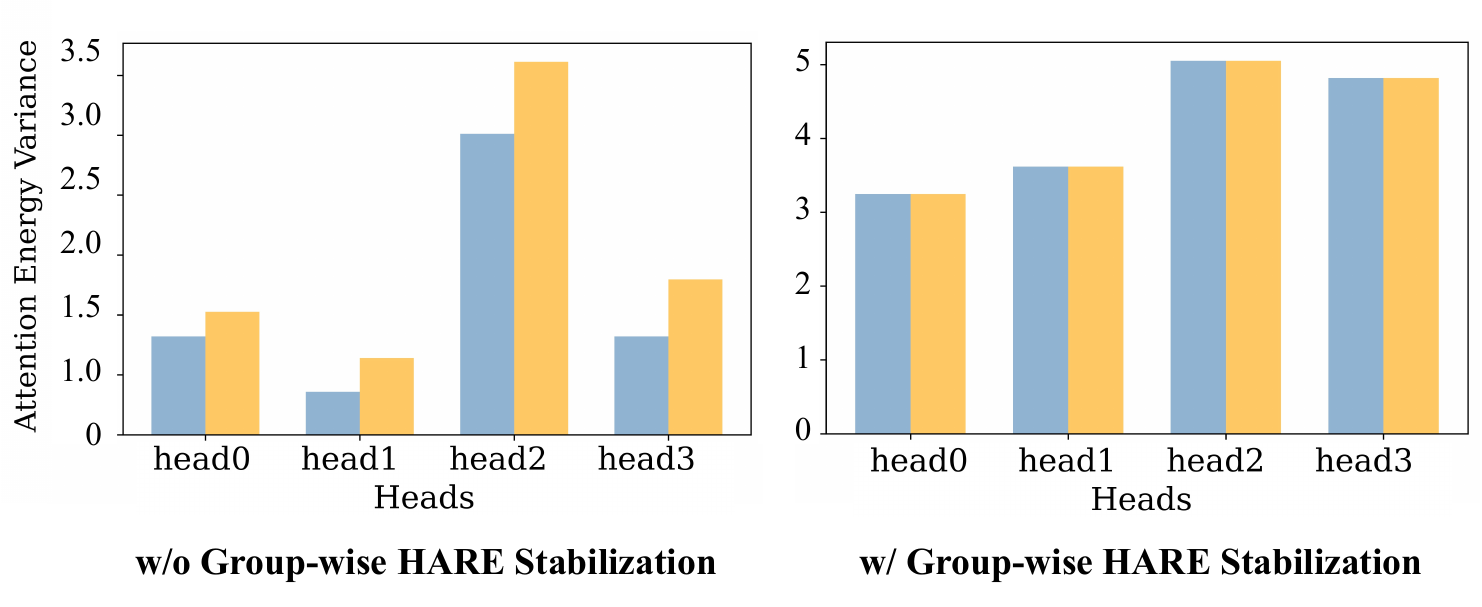}
   \caption{Effect of HARE stabilization on head-wise attention-response energy across samples. Without HARE stabilization, the \textcolor{myblue}{accurately predicted batch} consistently exhibits smaller variance than the \textcolor{myorange}{inaccurately predicted batch}. With HARE stabilization, the variance gap between the \textcolor{myblue}{accurately predicted batch} and the \textcolor{myorange}{inaccurately predicted batch} becomes much smaller, showing that HARE effectively regularizes cross-sample attention-response energy and promotes more stable attention behavior.}
   \label{fig:ablation2}
\end{figure}

% \begin{table}[htbp]
% \centering
% \renewcommand{\arraystretch}{1}
% \setlength{\tabcolsep}{2.2pt}
% \caption{Ablation study of group partition.}
% {%
% \begin{tabular}{|l|cccccc|}
% \hline
% \multirow{2}{*}{Method} & \multicolumn{6}{c|}{\cellcolor{gray!20}SEVIR (Unimodal)} \\
% \cline{2-7}
%  & \textuparrow CSI-M & \textuparrow CSI-181 & \textuparrow CSI-219 & \textuparrow HSS & \textdownarrow LPIPS  & \textuparrow SSIM  \\ 
% \hline

% HARECast & \textbf{0.344} & \textbf{0.183} & \textbf{0.098} & \textbf{0.437} & \textbf{0.211} & \textbf{0.692}\\
% w/o Grouping & 0.318 & 0.155 & 0.070 & 0.413 & 0.232 & 0.652 \\
% \hline
% \end{tabular}%
% }
% \label{tab:ablation_3}
% \end{table}

\noindent\textbf{Effectiveness of Group Partition.}
We compare the proposed head grouping design with a variant without grouping, where all heads are stabilized jointly using a shared target. As shown in Table~\ref{tab:ablation_1} lower block, grouping heads into strong-, intermediate-, and weak-response categories consistently improves all evaluation metrics, with especially clear gains on stricter threshold-based CSI metrics. This suggests that head grouping enables more precise energy control than treating all heads uniformly.

\subsection{Sensitivity Analysis}

\noindent\textbf{Effect of Batch Size.}
Since $\mathcal{L}_{\text{hare}}$ is built from batch-level energy statistics, the batch size directly affects the quality of the group mean estimate $\mu_g$ used as the stabilization target. With a smaller batch, the estimated target is computed from fewer samples and therefore becomes more sensitive to sample-specific fluctuations, making it less representative of the underlying response level of the group. As shown in Table~\ref{tab:ablation_2} upper block, this leads to consistent degradation across all metrics as the batch size decreases from 8 to 2. The effect is also visible in Figure~\ref{fig:ablation3}: larger batches yield smoother optimization of $\mathcal{L}_{\text{hare}}$, whereas smaller batches introduce noticeably stronger oscillations. These results suggest that HARE stabilization relies on sufficiently stable batch statistics; otherwise, noisy target estimates weaken the intended alignment of attention-response energy across samples and reduce the effectiveness of the stabilization mechanism.

%Since the batch-level energy statistics is used in group-wise HARE stabilization loss $\mathcal{L}_\text{hare}$, the batch size controls how much diversity in attention energy the model can see at each update. We therefore vary the batch size to assess its impact on the encoder, specifically the energy-routing mechanism. 
%To maintain training stability, the main batch size remains while a reduced batch is used solely to calculate the target energy  $\mu_{g}$, when computing the $\mathcal{L}_\text{hare}$ loss. 
%As shown in Table~\ref{tab:ablation_2}, reducing the batch size leads to a clear performance drop in CSI, HSS, LPIPS, and SSIM. As the batch size decreases, the diversity of attention energies within each batch shrinks, making the batch-level energy targets $\mu_{g}$ less representative and weakening the contrast between high and low energy samples. Consequently, HARE stabilization becomes less effective at stabilizing attention behaviors, which in turn degrades the predictions. Additionally, the left subplot in Figure~\ref{fig:ablation3} shows the training dynamics of the stabilization loss under different batch sizes. With batch size 8, shown as the pink line, the loss decreases smoothly and steadily, indicating stable energy-routing optimization. In contrast, batch size 4, shown as the blue line, exhibits noticeably larger fluctuations, reflecting noisier attention energy estimates and less reliable target energy $\mu_{g}$ statistics. This further supports that a sufficiently large batch is important for stable learning of the energy regulation mechanism.

\begin{table}[htbp]
\centering
\renewcommand{\arraystretch}{1}
\setlength{\tabcolsep}{3pt}
\caption{ Effect of batch size (BS) for computing HARE target statistics (upper block) and loss weights $(\lambda_1,\lambda_2,\lambda_3)$ for HARECast on SEVIR (lower block).}
{%
\begin{tabular}{|l|cccccc|}
\hline
\multirow{2}{*}{Method} & \multicolumn{6}{c|}{\cellcolor{gray!20}SEVIR (Unimodal)} \\
\cline{2-7}
 & \textuparrow CSI-M & \textuparrow CSI-181 & \textuparrow CSI-219 & \textuparrow HSS & \textdownarrow LPIPS  & \textuparrow SSIM  \\ 

\hline
BS = 2 & 0.318 & 0.156 & 0.063 & 0.417 & 0.228 & 0.671 \\
BS = 4 & 0.326 & 0.159 & 0.078 & 0.421 & 0.217 & 0.676 \\
BS = 8 & \textbf{0.344} & \textbf{0.183} & \textbf{0.098} & \textbf{0.437} & \textbf{0.211} & \textbf{0.692}\\
\hline

 (5,1,1)  & 0.309 & 0.126 & 0.075 & 0.402 & 0.254 & 0.669\\ 
(1,5,1)   & 0.313 & 0.149 & 0.079 & 0.415 & 0.244 & 0.650 \\ 
(1,1,5)  & \textbf{0.344} & \textbf{0.183} & \textbf{0.098} & \textbf{0.437} & \textbf{0.211} & \textbf{0.692}\\ 
\hline

\end{tabular}%
}
\label{tab:ablation_2}
\end{table}

\begin{figure}[t]
  \centering
   \includegraphics[width=\linewidth]{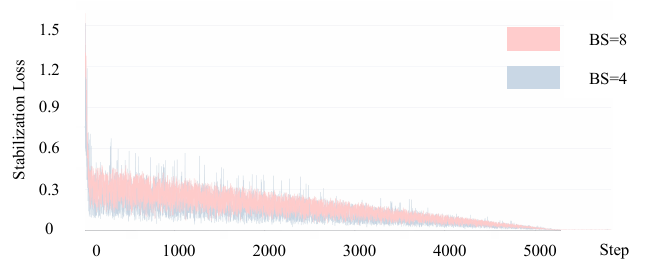}
   \caption{Training dynamics of $\mathcal{L}_{\text{hare}}$ w/ different batch sizes.}
   \label{fig:ablation3}
\end{figure}

\noindent\textbf{Sensitivity to Loss Weights.}
We evaluate the sensitivity of HARECast to the loss weights $\lambda_1$, $\lambda_2$, and $\lambda_3$, which control $\mathcal{L}_{\text{recon}}$, $\mathcal{L}_{\text{hare}}$, and $\mathcal{L}_{\text{diff}}$, respectively. As shown in Table~\ref{tab:ablation_2} lower block, the model remains reasonably stable across different weight configurations, while $(\lambda_1,\lambda_2,\lambda_3)=(1,1,5)$ gives the best overall performance. In contrast, assigning a relatively smaller weight to the diffusion objective leads to consistent degradation across metrics. This suggests that, although HARECast is not overly sensitive to moderate changes in the loss weights, sufficiently emphasizing the forecasting objective remains important for the best performance.

\subsection{Model Complexity}
We analyze the complexity of HARECast on SEVIR for 100-minute forecasting. HARECast has 90.88M parameters and requires 13.51 TFLOPs, which is comparable to recent high-performing methods such as AlphaPre and DuoCast in model size, while being substantially more efficient than DiffCast in FLOPs. Although it is not the lightest model, its cost remains practical: on a single NVIDIA RTX 4090 GPU, HARECast takes 0.86 seconds per sample to generate 20 future frames. This demonstrates that HARECast achieves strong forecasting performance with practical inference efficiency.

\section{Conclusion}

We identify cross-sample instability of attention-response energy as an important and previously underexplored source of instability for precipitation nowcasting. Based on this insight, we propose \textbf{HARECast}, a head-wise attention response energy-regulated framework that stabilizes attention responses through group-wise regularization. Extensive experiments show that this principle leads to strong empirical gains on benchmark datasets. %More broadly, our findings suggest that, beyond improving model expressiveness, explicitly controlling the stability of internal attention responses is a promising direction for building more reliable nowcasting systems.

\begin{acks}
This work was supported by Ant Group. 
\end{acks}

%\begin{acks}
%To Robert, for the bagels and explaining CMYK and color spaces.
%\end{acks}

\bibliographystyle{ACM-Reference-Format}
\bibliography{sample-base}

\clearpage

\appendix

\setcounter{figure}{8}
\setcounter{table}{8}

\section{Further Evidence of Motivation}

We argue that the \emph{stability of attention responses across samples} has not been adequately examined in precipitation nowcasting. In this task, samples can vary markedly in precipitation intensity, spatial coverage, motion patterns, and cross-source consistency~\cite{yu2025integrating,zhang2023skilful}. Such heterogeneity can lead to substantial cross-sample variation in attention responses, especially in multimodal forecasting where inputs from different sources introduce additional variability. To support this claim, we first provide additional statistical evidence of cross-sample variability in both uni- and multimodal settings. We then empirically show that this increased variance can directly induce large fluctuations in attention responses.

Table~\ref{tab:dataset_variance} provides statistical evidence for this phenomenon. As a reference, we first report the cross-sample variance of several general image datasets, including ImageNet~\cite{deng2009imagenet} and COCO~\cite{lin2014microsoft}. Notably, precipitation data exhibits cross-sample variance of a comparable scale, and in some cases even larger, than these general vision datasets. This is particularly striking because precipitation nowcasting is not a generic visual recognition problem: its variability must be modeled under strict spatiotemporal continuity and physical consistency constraints. As a result, cross-sample variance in meteorological data is inherently more challenging for sequential forecasting.

This effect becomes even more pronounced in the multimodal setting. For SEVIR~\cite{veillette2020sevir}, multimodal inputs exhibit substantially higher cross-sample variance than their unimodal counterparts. The variance increases from $19.1\times10^{-3}$ to $51.9\times10^{-3}$, far exceeding the variance of ImageNet and COCO. The result indicates that introducing additional modalities amplifies sample heterogeneity, which in turn makes stable attention modeling more difficult. This observation further supports our motivation that attention-response stability is especially important in multimodal precipitation nowcasting.

We next examine whether this input variability is reflected in the internal behavior of attention-based forecasting models. Specifically, we evaluate each predicted batch on PiMMNet~\cite{yu2025pimmnet} using CSI-M, and divide all batches into two groups: accurately predicted batches, whose CSI-M is above the average over all predicted batches, and inaccurately predicted batches, whose CSI-M is below that average. For each group, we visualize the cross-sample variance of head-wise attention-response energy across layers and heads. As shown in Figure~1(a) of the main manuscript, accurately predicted batches consistently exhibit smaller attention-energy variance, whereas inaccurately predicted batches show much larger fluctuations across heads and layers.

\begin{table}[htbp]
\centering
\setlength{\tabcolsep}{11pt}
\caption{The cross-sample variance in different datasets and settings.}
{%
\begin{tabular}{|l|c|}
\hline  Dataset & Cross-Sample Variance ($10^{-3}$) \\
\hline ImageNet & 18.4 \\
% \hline CIFAR100 & 21.6 \\
\hline COCO & 14.5 \\
\hline SEVIR (Unimodal) & 19.1 \\
\hline SEVIR (Multimodal) & 51.9 \\
% \hline MeteoNet (Unimodal) & 2.6 \\
% \hline MeteoNet (Multimodal) & 3.8 \\
\hline
\end{tabular}%
}
\label{tab:dataset_variance}
\end{table}

Finally, we provide a theoretical analysis showing that cross-sample variability in the input can propagate through self-attention and induce variance in attention responses, which further enlarges a lower bound on prediction error. Together, these results indicate that reliable nowcasting requires not only expressive attention, but also \emph{stable attention responses across samples}. Detailed proofs are provided in the next section.

\section{Foundations of Attention Response Energy}
In this section, we provide proofs for the results in Section 3.2 From Attention Variability to Prediction Error in the manuscript.

\noindent\textbf{Notation.}
For a random vector or matrix $\mathbf{Z}$, its total variance is defined by
\begin{equation}
\mathrm{Var}(\mathbf{Z})
:=
\mathbb{E}\bigl\|\mathbf{Z}-\mathbb{E}[\mathbf{Z}]\bigr\|^2.
\label{eq:supp_total_variance}
\end{equation}
Equivalently,
\begin{equation}
\mathrm{Var}(\mathbf{Z})
=
\mathrm{tr}\!\bigl(\mathrm{Cov}(\mathrm{vec}(\mathbf{Z}))\bigr).
\end{equation}
For scalar random variables, this reduces to the usual variance.

\begin{assumption}[Variance non-degeneracy of self-attention]
\label{assump:supp_attn_var}
Let $F(\mathbf{X})$ denote the self-attention mapping induced by the queries and keys, where $\mathbf{X}$ is the input feature map. We assume that, on the data manifold, the self-attention mapping does not collapse cross-sample variability. Specifically, there exists a constant $c_F>0$ such that
\begin{equation}
\mathrm{Var}\!\bigl(F(\mathbf{X})\bigr)
\ge
c_F^2\,\mathrm{Var}(\mathbf{X}).
\label{eq:supp_attn_var_assumption}
\end{equation}
\end{assumption}

\begin{lemma}[Variance propagation through the prediction head]
\label{lem:supp_linear_var}
Let
\begin{equation}
\hat{\mathbf{Y}}=\mathbf{W}F(\mathbf{X})+\mathbf{b},
\label{eq:supp_linear_predictor}
\end{equation}
and assume that the linear head is non-degenerate with minimum singular value
\begin{equation}
\sigma_{\min}(\mathbf{W})\ge c_G>0.
\label{eq:supp_head_nondegenerate}
\end{equation}
Then
\begin{equation}
\mathrm{Var}(\hat{\mathbf{Y}})
\ge
c_G^2\,\mathrm{Var}(F(\mathbf{X}))
\ge
c_G^2 c_F^2\,\mathrm{Var}(\mathbf{X}).
\label{eq:supp_linear_var_bound}
\end{equation}
\end{lemma}

\begin{proof}
Since adding a constant does not affect variance,
\begin{equation}
\mathrm{Var}(\hat{\mathbf{Y}})
=
\mathrm{Var}\!\bigl(\mathbf{W}F(\mathbf{X})\bigr).
\end{equation}
Let
\begin{equation}
\mathbf{Z}:=F(\mathbf{X})-\mathbb{E}[F(\mathbf{X})].
\end{equation}
Then
\begin{equation}
\hat{\mathbf{Y}}-\mathbb{E}[\hat{\mathbf{Y}}]
=
\mathbf{W}\mathbf{Z},
\end{equation}
and therefore
\begin{equation}
\mathrm{Var}(\hat{\mathbf{Y}})
=
\mathbb{E}\|\mathbf{W}\mathbf{Z}\|^2.
\label{eq:supp_var_head_start}
\end{equation}

For any matrix $\mathbf{U}$, the singular value bound gives
\begin{equation}
\|\mathbf{W}\mathbf{U}\|
\ge
\sigma_{\min}(\mathbf{W})\,\|\mathbf{U}\|.
\end{equation}
Applying this to $\mathbf{U}=\mathbf{Z}$ and squaring yields
\begin{equation}
\|\mathbf{W}\mathbf{Z}\|^2
\ge
\sigma_{\min}(\mathbf{W})^2\,\|\mathbf{Z}\|^2
\ge
c_G^2\,\|\mathbf{Z}\|^2.
\end{equation}
Taking expectation and using \eqref{eq:supp_var_head_start},
\begin{equation}
\mathrm{Var}(\hat{\mathbf{Y}})
\ge
c_G^2\,\mathbb{E}\|\mathbf{Z}\|^2
=
c_G^2\,\mathrm{Var}(F(\mathbf{X})).
\end{equation}
Combining this with Assumption~\ref{assump:supp_attn_var},
\begin{equation}
\mathrm{Var}(F(\mathbf{X}))
\ge
c_F^2\,\mathrm{Var}(\mathbf{X}),
\end{equation}
gives
\begin{equation}
\mathrm{Var}(\hat{\mathbf{Y}})
\ge
c_G^2 c_F^2\,\mathrm{Var}(\mathbf{X}).
\end{equation}
\end{proof}

\begin{lemma}[Lower bound on prediction error]
\label{lem:supp_mse_lower}
Define the mean squared error (MSE) as
\begin{equation}
\mathrm{MSE}
=
\mathbb{E}\bigl\|\mathbf{Y}-\hat{\mathbf{Y}}\bigr\|^2.
\label{eq:supp_mse_def}
\end{equation}
Then
\begin{equation}
\mathrm{MSE}
\ge
\bigl\|\mathbb{E}[\hat{\mathbf{Y}}]-\mathbb{E}[\mathbf{Y}]\bigr\|^2
+
\left(
\mathrm{Var}(\hat{\mathbf{Y}})^{1/2}
-
\mathrm{Var}(\mathbf{Y})^{1/2}
\right)^2.
\label{eq:supp_mse_lower_bound}
\end{equation}
\end{lemma}

\begin{proof}
Write
\begin{equation}
\boldsymbol{\mu}_{\hat Y}:=\mathbb{E}[\hat{\mathbf{Y}}],
\qquad
\boldsymbol{\mu}_{Y}:=\mathbb{E}[\mathbf{Y}].
\end{equation}
Then
\begin{equation}
\mathbf{Y}-\hat{\mathbf{Y}}
=
(\mathbf{Y}-\boldsymbol{\mu}_{Y})
-
(\hat{\mathbf{Y}}-\boldsymbol{\mu}_{\hat Y})
+
(\boldsymbol{\mu}_{Y}-\boldsymbol{\mu}_{\hat Y}).
\end{equation}
Taking squared norm and expectation,
\begin{align}
\mathrm{MSE}
&=
\mathbb{E}\bigl\|
(\mathbf{Y}-\boldsymbol{\mu}_{Y})
-
(\hat{\mathbf{Y}}-\boldsymbol{\mu}_{\hat Y})
+
(\boldsymbol{\mu}_{Y}-\boldsymbol{\mu}_{\hat Y})
\bigr\|^2 \notag\\
&=
\|\boldsymbol{\mu}_{Y}-\boldsymbol{\mu}_{\hat Y}\|^2
+
\mathbb{E}\bigl\|
(\mathbf{Y}-\boldsymbol{\mu}_{Y})
-
(\hat{\mathbf{Y}}-\boldsymbol{\mu}_{\hat Y})
\bigr\|^2,
\label{eq:supp_bias_var_expand}
\end{align}
because the cross term vanishes after expectation.

Next, expand the second term:
\begin{align}
\mathbb{E}\bigl\|
(\mathbf{Y}-\boldsymbol{\mu}_{Y})
-
(\hat{\mathbf{Y}}-\boldsymbol{\mu}_{\hat Y})
\bigr\|^2
&=
\mathrm{Var}(\mathbf{Y})
+
\mathrm{Var}(\hat{\mathbf{Y}})
-
2\,\mathbb{E}\!\left[
\langle
\mathbf{Y}-\boldsymbol{\mu}_{Y},
\hat{\mathbf{Y}}-\boldsymbol{\mu}_{\hat Y}
\rangle
\right].
\end{align}
By Cauchy--Schwarz,
\begin{equation}
\mathbb{E}\!\left[
\langle
\mathbf{Y}-\boldsymbol{\mu}_{Y},
\hat{\mathbf{Y}}-\boldsymbol{\mu}_{\hat Y}
\rangle
\right]
\le
\sqrt{\mathrm{Var}(\mathbf{Y})\,\mathrm{Var}(\hat{\mathbf{Y}})}.
\end{equation}
Hence
\begin{equation}
\mathbb{E}\bigl\|
(\mathbf{Y}-\boldsymbol{\mu}_{Y})
-
(\hat{\mathbf{Y}}-\boldsymbol{\mu}_{\hat Y})
\bigr\|^2
\ge
\mathrm{Var}(\mathbf{Y})
+
\mathrm{Var}(\hat{\mathbf{Y}})
-
2\sqrt{\mathrm{Var}(\mathbf{Y})\,\mathrm{Var}(\hat{\mathbf{Y}})}.
\end{equation}
The right-hand side equals
\begin{equation}
\left(
\mathrm{Var}(\hat{\mathbf{Y}})^{1/2}
-
\mathrm{Var}(\mathbf{Y})^{1/2}
\right)^2.
\end{equation}
Substituting this into \eqref{eq:supp_bias_var_expand} proves Lemma~\ref{lem:supp_mse_lower}.
\end{proof}

\begin{theorem}[MSE lower bound induced by the attention response and the input variability]
\label{thm:supp_main}
Assume that the linear prediction head is non-degenerate with minimum singular value
$\sigma_{\min}(\mathbf{W})\ge c_G>0$.
Assume further that the attention response satisfies
$\mathrm{Var}(F(\mathbf{X})) \ge c_F^2 \mathrm{Var}(\mathbf{X})$
for some constant $c_F>0$. In addition, we assume that 
\begin{equation}
\mathrm{Var}(\mathbf{Y}) = \mathrm{Var}(\mathbf{X}),
\label{eq:supp_thm_matched}
\end{equation}
since \(\mathbf{X}\) and \(\mathbf{Y}\) are drawn from the same precipitation event sequence and therefore are assumed to have close variability. 
If $c_G c_F>1$, then
\begin{equation}
\mathrm{MSE}
\ge
\bigl\|\mathbb{E}[\hat{\mathbf{Y}}]-\mathbb{E}[\mathbf{Y}]\bigr\|^2
+
\left(
c_G\,\mathrm{Var}(F(\mathbf{X}))^{1/2}
-
\mathrm{Var}(\mathbf{Y})^{1/2}
\right)^2,
\label{eq:supp_thm_bound_F}
\end{equation}
and consequently
\begin{equation}
\mathrm{MSE}
\ge
\bigl\|\mathbb{E}[\hat{\mathbf{Y}}]-\mathbb{E}[\mathbf{Y}]\bigr\|^2
+
\left(c_G-\frac{1}{c_F}\right)^2 \mathrm{Var}(F(\mathbf{X})),
\label{eq:supp_thm_bound_F2}
\end{equation}
as well as
\begin{equation}
\mathrm{MSE}
\ge
\bigl\|\mathbb{E}[\hat{\mathbf{Y}}]-\mathbb{E}[\mathbf{Y}]\bigr\|^2
+
(c_G c_F - 1)^2 \mathrm{Var}(\mathbf{Y}).
\label{eq:supp_thm_bound_X}
\end{equation}
In particular, if the predictor is unbiased, i.e.,
$\mathbb{E}[\hat{\mathbf{Y}}]=\mathbb{E}[\mathbf{Y}]$,
then the last two bounds reduce to
\begin{equation}
\mathrm{MSE}\ge \left(c_G-\frac{1}{c_F}\right)^2 \mathrm{Var}(F(\mathbf{X})),
\qquad
\mathrm{MSE}\ge (c_G c_F - 1)^2 \mathrm{Var}(\mathbf{X}).
\label{eq:supp_unbiased_case}
\end{equation}
\end{theorem}

\begin{proof}
By Lemma~\ref{lem:supp_linear_var},
\begin{equation}
\mathrm{Var}(\hat{\mathbf{Y}})
\ge
c_G^2\,\mathrm{Var}(F(\mathbf{X})).
\end{equation}
Taking square roots,
\begin{equation}
\mathrm{Var}(\hat{\mathbf{Y}})^{1/2}
\ge
c_G\,\mathrm{Var}(F(\mathbf{X}))^{1/2}.
\end{equation}
Also, Assumption~\ref{assump:supp_attn_var} and $\mathrm{Var}(\mathbf{Y})=\mathrm{Var}(\mathbf{X})$ imply
\begin{equation}
c_G\,\mathrm{Var}(F(\mathbf{X}))^{1/2}
\ge
c_G c_F\,\mathrm{Var}(\mathbf{X})^{1/2}
=
c_G c_F\,\mathrm{Var}(\mathbf{Y})^{1/2}.
\end{equation}
Since $c_G c_F>1$, we obtain
\begin{equation}
c_G\,\mathrm{Var}(F(\mathbf{X}))^{1/2}
>
\mathrm{Var}(\mathbf{Y})^{1/2}.
\end{equation}
Therefore, on the relevant range, the function
\begin{equation}
u \mapsto (u-\mathrm{Var}(\mathbf{Y})^{1/2})^2
\end{equation}
is increasing. Hence
\begin{equation}
\left(
\mathrm{Var}(\hat{\mathbf{Y}})^{1/2}
-
\mathrm{Var}(\mathbf{Y})^{1/2}
\right)^2
\ge
\left(
c_G\,\mathrm{Var}(F(\mathbf{X}))^{1/2}
-
\mathrm{Var}(\mathbf{Y})^{1/2}
\right)^2.
\end{equation}
Combining this with Lemma~\ref{lem:supp_mse_lower} yields \eqref{eq:supp_thm_bound_F}.

Next, since
\begin{equation}
\mathrm{Var}(F(\mathbf{X}))
\ge
c_F^2\,\mathrm{Var}(\mathbf{X})
=
c_F^2\,\mathrm{Var}(\mathbf{Y}),
\end{equation}
we have
\begin{equation}
\mathrm{Var}(\mathbf{Y})^{1/2}
\le
\frac{1}{c_F}\,\mathrm{Var}(F(\mathbf{X}))^{1/2}.
\end{equation}
Thus
\begin{equation}
c_G\,\mathrm{Var}(F(\mathbf{X}))^{1/2}
-
\mathrm{Var}(\mathbf{Y})^{1/2}
\ge
\left(c_G-\frac{1}{c_F}\right)\mathrm{Var}(F(\mathbf{X}))^{1/2}.
\end{equation}
Squaring both sides gives \eqref{eq:supp_thm_bound_F2}.

Similarly,
\begin{equation}
\mathrm{Var}(F(\mathbf{X}))^{1/2}
\ge
c_F\,\mathrm{Var}(\mathbf{X})^{1/2}
=
c_F\,\mathrm{Var}(\mathbf{Y})^{1/2},
\end{equation}
so
\begin{equation}
c_G\,\mathrm{Var}(F(\mathbf{X}))^{1/2}
-
\mathrm{Var}(\mathbf{Y})^{1/2}
\ge
(c_G c_F-1)\,\mathrm{Var}(\mathbf{Y})^{1/2}.
\end{equation}
Squaring again gives \eqref{eq:supp_thm_bound_X}.
\end{proof}

\noindent\textbf{Unbiased case.}
If the predictor is unbiased, i.e.
\begin{equation}
\mathbb{E}[\hat{\mathbf{Y}}]=\mathbb{E}[\mathbf{Y}],
\end{equation}
then the bias term vanishes in \eqref{eq:supp_thm_bound_F}--\eqref{eq:supp_thm_bound_X}, and we obtain
\begin{equation}
\mathrm{MSE}
\ge
\left(c_G-\frac{1}{c_F}\right)^2 \mathrm{Var}(F(\mathbf{X})),
\end{equation}
and
\begin{equation}
\mathrm{MSE}
\ge
(c_G c_F-1)^2 \mathrm{Var}(\mathbf{X}).
\end{equation}

\section{Dataset Details}
In this section, we provide extra detailed descriptions of the two commonly used benchmark datasets, SEVIR~\cite{veillette2020sevir} and MeteoNet~\cite{larvor2020meteo}, including their spatial and temporal resolutions, threshold settings, data splits, and preprocessing procedures.

SEVIR is an annotated dataset that temporally and spatially aligns multiple meteorological products, including visible satellite imagery, infrared channels (mid-level water vapor and clean longwave window), NEXRAD radar mosaics of vertically integrated liquid (VIL), and ground-based lightning detections. Each event consists of a 4-hour sequence of images sampled every 5 minutes over a $384\text{ km} \times 384\text{ km}$ region across the continental United States. Radar data have a spatial resolution of 1 km, whereas satellite data are provided at 2 km resolution.  In our experiments, we use only the radar (VIL) and satellite IR107 channels and focus on precipitation events in each sequence. All frames are rescaled to the range [0, 255] and binarized at thresholds [16, 74, 133, 160, 181, 219] for computing CSI and HSS. 
For single-modality comparison, we follow AlphaPre~\cite{Lin_2025_CVPR} setting: The data are split into training, validation, and test sets using January 1, 2019 and June 1, 2019 as temporal cutoffs. Each event consists of 25 consecutive frames sampled, where 5 observed frames (25 minutes) are used to predict the next 20 VIL frames (100 minutes). All inputs are downsampled to a spatial resolution of $128 \times 128$ to reduce computational overhead. For multimodal comparison, we adopt PiMMNet~\cite{yu2025pimmnet} setting: It predicts 3-hour precipitation evolution (36 frames) from the preceding 1-hour observations (12 frames). We maintain the original temporal resolution while downscaling the spatial resolution to $128 \times 128$ for efficiency.

MeteoNet is a multimodal dataset containing time series of satellite and radar imagery, numerical weather model outputs, and ground observations. It covers a $550\text{ km} \times 550\text{ km}$ region over northwestern and southeastern France and spans three years (2016–2018), with radar reflectivity recorded every 5 minutes and satellite infrared data every hour. Radar has $0.01^{\circ}$ ($\approx$1.11 km) resolution and satellite has $0.03^{\circ}$ ($\approx$3.33 km) resolution. In this work, we use the radar reflectivity fields together with the satellite IR108 channel. Because the satellite data have a much lower temporal resolution, we use only a single IR frame aligned with radar input and replicate it $T_I$ times to match the radar sequence length. Following AlphaPre~\cite{Lin_2025_CVPR}: We split radar sequences from 2016–2018 into training, validation, and test sets using January 1, 2018 and June 1, 2018 as temporal cutoffs. To reduce computational cost, all inputs are downsampled to $128 \times 128$ resolution. For CSI and HSS evaluation, we adopt thresholds [12, 18, 24, 32], consistent with AlphaPre~\cite{Lin_2025_CVPR}.

\section{Additional Results and Analysis}
In this section, we provide an extra analysis on the design of framework loss, the effectiveness of our model design and the computational complexity.

\subsection{Attention vs. Convolution} 

Consistent with this motivation and recent practice~\cite{yu2024diffcast, yu2025integrating}, we use attention to enable sample-adaptive non-local modeling, which applies to both unimodal and multimodal observations.
To verify the superiority of attention mechanism, we add ablations shown in Table~\ref{tab:attn_vs_conv}, comparing vanilla convolution (Conv), DuoCast Air-mass convolution (Air-mass)~\cite{wen2026duocast}, vanilla attention (Attn), and our HARE-stabilized attention. Vanilla attention outperforms vanilla and Air-mass convolution, and HARE further improves vanilla attention, showing that the gain comes from stabilizing adaptive attention responses.

\begin{table}[htbp]
\centering
\renewcommand{\arraystretch}{1}
\setlength{\tabcolsep}{3.5 pt}
\caption{Comparison of attention and convolution methods on the multimodal SEVIR dataset.}
{%
\begin{tabular}{|l|cccccc|}
\hline
\multirow{2}{*}{Method} & \multicolumn{6}{c|}{\cellcolor{gray!20}SEVIR (Multimodal)} \\
\cline{2-7}
 & \textuparrow CSI-M & \textuparrow CSI$_{4}$ & \textuparrow CSI$_{16}$ & \textuparrow CSI-219 & \textuparrow HSS & \textdownarrow LPIPS  \\ 
\hline
 Conv  & 0.206 & 0.243 & 0.344 & 0.017 & 0.261 & 0.314\\ 
 Air-mass  & 0.219 & 0.259 & 0.360 & 0.022 & 0.281 & 0.289\\ 
Attn  & 0.230 & 0.267 & 0.376 & 0.031 & 0.297 & 0.255\\ 
HARECast  & \textbf{0.254} & \textbf{0.316} & \textbf{0.421} & \textbf{0.053} & \textbf{0.315} & \textbf{0.246}\\ 
\hline
\end{tabular}%
}
\label{tab:attn_vs_conv}
\end{table}

\subsection{Empirical Exam on Theoretical Assumptions ($c_F > 1, c_G >1$)} 
We further conduct an empirical sanity check on 20 random batches, shown in Table~\ref{tab:empirical_exam}. The average inter-sample variance increases from $7.4{\times}10^{-4}$ at the first attention layer to $8.88{\times}10^{-3}$ at the second attention layer and $4.606{\times}10^{-2}$ at the final layer. This is consistent with Assumption 1 and provides empirical support for the variance-propagation premise in Lemma 1 and Theorem 1. 

\begin{table}[!htbp]
\centering
\caption{Empirical examination of inter-sample variance propagation across attention layers on 20 randomly sampled batches.}
\renewcommand{\arraystretch}{1.0}
\setlength{\tabcolsep}{18pt}
\begin{tabular}{|l|c|}
\hline
Layer & \cellcolor{gray!20}Average Inter-sample Variance \\
\hline
Attn. Layer 1 & 0.00074 \\
Attn. Layer 2 & 0.00888 \\
Final Layer & 0.04606 \\
\hline
\end{tabular}
\label{tab:empirical_exam}
\end{table}

\subsection{Empirical Exam on Theoretical Assumption $\mathrm{Var}(\mathbf{X}) = \mathrm{Var}(\mathbf{Y})$} 

Since $\mathbf{X}$ and $\mathbf{Y}$ are consecutive normalized radar observations from the same event, we use $\mathrm{Var}(\mathbf{X}) \approx \mathrm{Var}(\mathbf{Y})$ as a mild approximation. 
As shown in Table~\ref{tab:sevir_x_y}, their inter-sample variances are close on SEVIR, with a ratio of $1.06$, empirically supporting this approximation.

\begin{table}[!htbp]
\centering
\caption{Empirical comparison of the inter-sample variances of consecutive normalized radar observations $\mathbf{X}$ and $\mathbf{Y}$ on SEVIR.}
\renewcommand{\arraystretch}{0.95}
\setlength{\tabcolsep}{14pt}
\begin{tabular}{|l|ccc|}
\hline
Data & \cellcolor{gray!20}Var($\mathbf{X}$) & \cellcolor{gray!20}Var($\mathbf{Y}$) & \cellcolor{gray!20}Ratio \\
\hline
SEVIR & 0.002325 & 0.002465 & 1.060 \\
\hline
\end{tabular}
\label{tab:sevir_x_y}
\end{table}

\subsection{Hyper-Parameters for Loss Weights}
We further analyze the training behavior of HARECast by examining the loss curves shown in Figure~\ref{fig:all_losses}, together with the complementary sensitivity study on different loss-weight configurations reported in Table 7 of the main manuscript. These results provide additional insight into both the optimization stability of the proposed framework and the robustness of its loss design. 

As shown in the figure, all loss terms decrease steadily over the course of training, indicating that the overall optimization process remains stable. In particular, the reconstruction loss, HARE stabilization loss, and diffusion loss are all well-behaved and do not exhibit conflicting or unstable dynamics, suggesting good compatibility among the three objectives. This observation supports the effectiveness of the proposed multi-term training objective and shows that the different components of HARECast can be optimized jointly in a balanced manner. 

\begin{figure}[t]
  \centering
   \includegraphics[width=\linewidth]{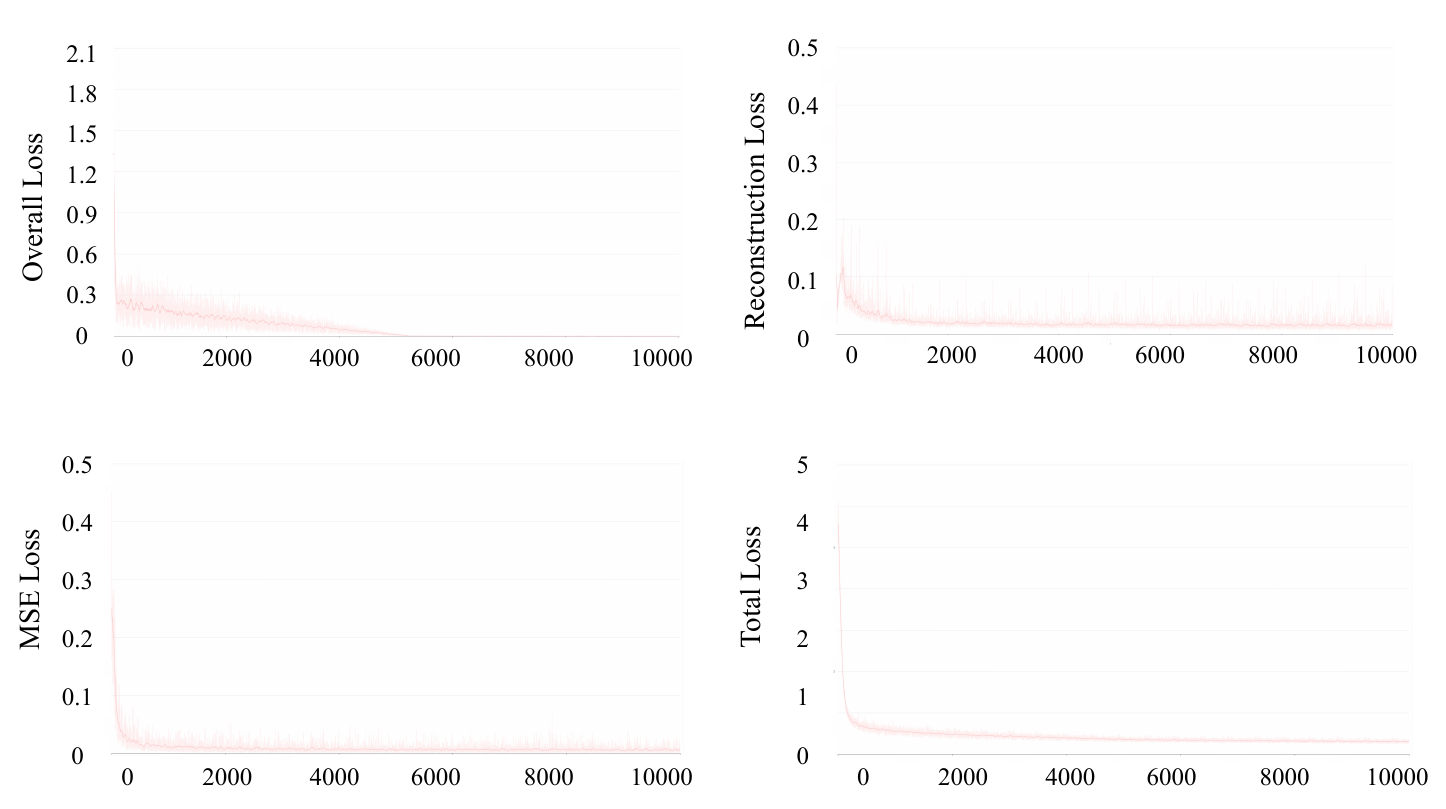}
   \caption{
   Training dynamics of the loss terms. All components of the objective exhibit stable, monotonic decreases during training, indicating a well-behaved optimization process.
   }
   \label{fig:all_losses}
\end{figure}

\subsection{Effectiveness of Group-wise HARE Stabilization in Multimodal Setting}

In addition to the ablation study reported in Table 4 of the main manuscript, we further evaluate the effect of the proposed HARE stabilization in the multimodal setting, as shown in Table~\ref{tab:ablation_multimodal}.

The table shows that the proposed group-wise HARE stabilization consistently improves performance in the multimodal setting. Removing G-HARE leads to degradation across all reported metrics. In particular, HARECast achieves clear gains on both moderate and heavy precipitation thresholds, indicating that attention stabilization is beneficial not only for overall forecasting quality but also for capturing more challenging high-intensity precipitation structures. The lower LPIPS score further suggests that HARE stabilization improves perceptual fidelity in the predicted precipitation fields. These results confirm that the proposed stabilization mechanism remains effective in the multimodal setting, where heterogeneous inputs introduce larger cross-sample variability and make stable attention modeling more challenging.

\begin{table}[htbp]
\centering
\renewcommand{\arraystretch}{1}
\setlength{\tabcolsep}{2.1pt}
\caption{Ablation study of HARE stabilization in multimodal setting. ``w/o G-HARE'' denotes removing the HARE stabilization module.}
{%
\begin{tabular}{|l|cccccc|}
\hline
\multirow{2}{*}{Method} & \multicolumn{6}{c|}{\cellcolor{gray!20}SEVIR (Multimodal)} \\
\cline{2-7}
 & \textuparrow CSI-M & \textuparrow CSI$_4$ & \textuparrow CSI$_{16}$ & \textuparrow CSI-219 & \textuparrow HSS  & \textdownarrow LPIPS  \\ 
\hline
HARECast & \textbf{0.254}& \textbf{0.316}& \textbf{0.421}& \textbf{0.053} & \textbf{0.315}& \textbf{0.246}  \\
w/o G-HARE & 0.221 &  0.279 & 0.377 & 0.037 & 0.289 & 0.257 \\
\hline
\end{tabular}%
}
\label{tab:ablation_multimodal}
\end{table}

\subsection{Effectiveness of Energy Head Group Partition}
In addition to the ablation study reported in Table 5 of the main manuscript, which validates the effectiveness of the proposed head-grouping design, we further visualize the attention maps of the $\mathcal H^{\uparrow}$ and $\mathcal H^{\circ} $ head groups shown in Figure~\ref{fig:head_vis}. 

The $\mathcal H^{\uparrow}$ heads produce much sharper and more localized attention patterns, with several bright, concentrated responses aligned with the main high-intensity precipitation structures in the input. This suggests that strong heads $\mathcal H^{\uparrow}$ are responsible for capturing salient, event-critical regions, such as intense rain cores and sharp structural transitions. In contrast, the $\mathcal H^{\circ}$ heads exhibit smoother and more diffuse attention maps, indicating that they respond more broadly to contextual and supporting information. Taken together, the visualization supports the motivation of the grouping design: different head groups play distinct roles.

\begin{figure}[t]
  \centering
   \includegraphics[width=\linewidth]{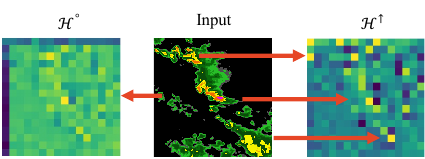}
   \caption{
   Visualization of attention maps from different head groups. Group $\mathcal H^{\circ}$ focuses on more diffuse responses of the contextual information, while group $\mathcal H^{\uparrow}$ exhibits sharper and more localized activations aligned with salient precipitation structures in the input.
   }
   \label{fig:head_vis}
\end{figure}

\subsection{Hyper-Parameter Sensitivity to Threshold $\alpha$}

In the main manuscript, the head-grouping threshold $\alpha$ is set to 0.75. To evaluate the robustness of this choice, we further conduct a sensitivity analysis over different values of $\alpha$, as reported in Table~\ref{tab:alpha_setting}.

As shown in Table~\ref{tab:alpha_setting}, $\alpha=0.75$ consistently yields the strongest performance across all evaluation metrics. When $\alpha=0.50$, the partition becomes overly conservative, so only a limited number of weak-response heads are assigned to the low-energy group, reducing the effectiveness of group-wise stabilization. On the other hand, $\alpha=0.90$ is too aggressive, causing many moderately activated heads to be grouped as low-energy and thereby weakening the discriminability between contextual and inactive heads. Overall, these results indicate that $\alpha=0.75$ achieves a more balanced head partition and is therefore better suited for the proposed HARE stabilization strategy.

\begin{table}[htbp]
\centering
\renewcommand{\arraystretch}{1}
\setlength{\tabcolsep}{2.1pt}
\caption{Effect of threshold $\alpha$ in Head Partition.}
{%
\begin{tabular}{|l|cccccc|}
\hline
\multirow{2}{*}{Method} & \multicolumn{6}{c|}{\cellcolor{gray!20}SEVIR (Unimodal)} \\
\cline{2-7}
 & \textuparrow CSI-M & \textuparrow CSI-181 & \textuparrow CSI-219 & \textuparrow HSS & \textdownarrow LPIPS  & \textuparrow SSIM  \\ 
\hline
$\alpha=0.75$ & \textbf{0.344} & \textbf{0.183} & \textbf{0.098} & \textbf{0.437} & \textbf{0.211} & \textbf{0.692}\\
$\alpha=0.50$ & 0.335 &  0.167 & 0.079 & 0.411 & 0.234 & 0.664 \\
$\alpha=0.90$ & 0.326 &  0.141 & 0.054 & 0.406 & 0.241 & 0.675 \\
\hline
\end{tabular}%
}
\label{tab:alpha_setting}
\end{table}

\subsection{Batch Size Sensitivity Analysis}
Larger batch sizes can stabilize the batch-level statistics during training and improve optimization. However, this sensitivity is confined to training. At inference, the HARE loss is not used, and HARECast performs standard forward prediction without computing or relying on $\mu_g$. Thus, the deployed model does not require a fixed inference batch size.

\subsection{Two-side Stabilization}
We further evaluate a double-sided HARE variant shown in Table~\ref{tab:two_side}. The results further support the effectiveness of our proposed one-sided formulation.

\begin{table}[!htbp]
\centering
\renewcommand{\arraystretch}{1}
\setlength{\tabcolsep}{1pt}
\caption{Comparison of one-sided and double-sided HARE stabilization on the unimodal SEVIR dataset.}
{%
\begin{tabular}{|l|cccccc|}
\hline
\multirow{2}{*}{Method} & \multicolumn{6}{c|}{\cellcolor{gray!20}SEVIR (Unimodal)} \\
\cline{2-7}
 & \textuparrow CSI-M & \textuparrow CSI-181 & \textuparrow CSI-219 & \textuparrow HSS & \textdownarrow LPIPS  & \textuparrow SSIM  \\ 

\hline
One-sided & \textbf{0.344} & \textbf{0.183} & \textbf{0.098} & \textbf{0.437} & \textbf{0.211} & \textbf{0.692} \\
Double-sided  & 0.336 & 0.175 & 0.088 & 0.424 & 0.221 & 0.679 \\
\hline
\end{tabular}%
}
\label{tab:two_side}
\end{table}

\subsection{More Ablation Study on Multimodal}
We add a multimodal w/o Grouping ablation. HARECast consistently improves all metrics, 
supporting the effectiveness of group-wise attention-energy stabilization for radar-satellite fusion, shown in Table~\ref{tab:ablation_multimodal_m}.

\begin{table}[!htbp]
\centering
\renewcommand{\arraystretch}{1}
\caption{Ablation study of G-HARE and group-wise attention-energy stabilization on the multimodal SEVIR dataset.}
\setlength{\tabcolsep}{1.5pt}
{%
\begin{tabular}{|l|cccccc|}
\hline
\multirow{2}{*}{Method} & \multicolumn{6}{c|}{\cellcolor{gray!20}SEVIR (Multimodal)} \\
\cline{2-7}
 & \textuparrow CSI-M & \textuparrow CSI$_4$ & \textuparrow CSI$_{16}$ & \textuparrow CSI-219 & \textuparrow HSS  & \textdownarrow LPIPS  \\ 
\hline
HARECast & \textbf{0.254}& \textbf{0.316}& \textbf{0.421}& \textbf{0.053} & \textbf{0.315}& \textbf{0.246}  \\
w/o G-HARE & 0.221 &  0.279 & 0.377 & 0.037 & 0.289 & 0.257 \\
w/o Grouping & 0.239 &  0.298 & 0.400 & 0.044 & 0.286 & 0.249 \\
\hline
\end{tabular}%
}
\label{tab:ablation_multimodal_m}
\end{table}

\subsection{Controlled Evidence Beyond Correlation}
To move beyond correlation, we add two controlled evaluations, shown in Table~\ref{tab:hc}. First, with the same inputs and fixed model weights, perturbing attention-response energy degrades performance, indicating that the energy scale is functionally relevant to forecasting rather than a passive correlate. Second, on hard cases (HC; pixels above the 181 threshold occupy $>1\%$ of the area), HARE consistently improves over w/o HARE, showing that the benefit remains when weather difficulty is explicitly considered. Together, these results support that stabilizing attention-response energy contributes to more robust forecasts.

\begin{table}[!htbp]
\centering
\caption{Controlled evaluation of attention-response energy perturbation and HARE performance on hard cases in the unimodal SEVIR dataset.}
\renewcommand{\arraystretch}{1}
\setlength{\tabcolsep}{0.8pt}
{%
\begin{tabular}{|l|cccccc|}
\hline
\multirow{2}{*}{Method} & \multicolumn{6}{c|}{\cellcolor{gray!20}SEVIR (Unimodal)} \\
\cline{2-7}
 & \textuparrow CSI-M & \textuparrow CSI-181 & \textuparrow CSI-219 & \textuparrow HSS & \textdownarrow LPIPS  & \textuparrow SSIM  \\ 

\hline
HARECast & \textbf{0.344} & \textbf{0.183} & \textbf{0.098} & \textbf{0.437} & \textbf{0.211} & \textbf{0.692} \\
Perturbed & 0.323 & 0.157 & 0.069 & 0.423 & 0.236 & 0.677 \\

\hline
HC w HARE & \textbf{0.137} & \textbf{0.061} & \textbf{0.040} & \textbf{0.139} & \textbf{0.549} & \textbf{0.169} \\ 
HC w/o HARE & 0.115 & 0.044 & 0.022 & 0.127 & 0.581 & 0.143 \\
\hline
\end{tabular}%
}
\label{tab:hc}
\end{table}

\subsection{Model Complexity}
Table~\ref{tab:complexities} provides an extra comparison of inference-time parameter counts and computational cost (TFLOPs) between HARECast and recent baselines on SEVIR (100-minute precipitation evolution), evaluated on a single NVIDIA RTX 4090 GPU. Compared with DiffCast, HARECast uses more parameters but requires substantially fewer TFLOPs, indicating better computational efficiency. Relative to AlphaPre, HARECast has a similar parameter scale but a higher computational cost. Even so, the overall computation remains well within the capability of modern hardware: for example, an RTX 4090 provides up to 82.6 TFLOPs in FP32 and 330.4 TFLOPs in FP16, whereas HARECast requires only 13.51 TFLOPs. These results suggest that HARECast achieves a favorable balance between model capacity and practical efficiency for precipitation nowcasting. In addition to computational cost, HARECast also exhibits competitive runtime efficiency. On the same hardware, it requires 0.86 seconds per sample to generate 20 forecast frames, corresponding to the 100-minute SEVIR horizon, whereas DiffCast requires 3.67 seconds per sample. This latency indicates that HARECast is compatible with real-time or near-real-time deployment on commodity GPUs.

\begin{table}[t]
\centering
\renewcommand{\arraystretch}{1}
\setlength{\tabcolsep}{2.8pt}
\caption{Analysis of model complexity with SoTA methods.}
\begin{tabularx}{\columnwidth}{
|>{\centering\arraybackslash}X
|>{\centering\arraybackslash}X
|>{\centering\arraybackslash}X
|>{\centering\arraybackslash}X|}
\hline
\multirow{2}{*}{Method} & \multicolumn{3}{c|}{Complexity}\\
\cline{2-4}
 & \# Param & TFLOPs & Year \\
\hline
Prediff & 135.24M & 2.80 & 2023\\
STRPM & 439.63M & 0.28 & 2022\\
SimVP  & 44.25M & 0.05 & 2022\\
Earth\-former & 34.61M & 0.04 & 2022\\
MAU & 20.13M & 0.09 & 2021\\
ConvGRU & 18.21M & 0.03 & 2017\\
PhyDNet & 11.80M & 0.08 & 2020 \\
DiffCast & 58.33M & 72.49 & 2024\\
FACL & 14.38M & 0.02 & 2024\\ 
AlphaPre & 89.03M & 1.56 & 2025\\
DuoCast & 94.18M & 0.30 & 2026\\
\hline
\rowcolor{lightblue}
HARECast & 90.88M & 13.51 & 2026\\
\hline
\end{tabularx}
\label{tab:complexities}
\end{table}

\subsection{The Statistical Significance of Improvement.}
We assess statistical significance on DuoCast, using paired tests on SEVIR for the metric CSI-M. As shown in Table~\ref{tab:t-test}, HARECast achieves a paired $t$-test $p=4.8{\times}10^{-16}$, a Wilcoxon signed-rank test $p=2.8{\times}10^{-23}$, and a paired bootstrap 95\% confidence interval of $[0.005, 0.006]$. The two tests reject the null hypothesis of no paired difference, and the bootstrap CI excludes zero, consistently supporting a statistically significant positive gain.

\begin{table}[!htbp]
\centering
\renewcommand{\arraystretch}{1.05}
\setlength{\tabcolsep}{4pt}
\caption{Statistical significance analysis of the performance difference between the compared methods using a paired $t$-test, Wilcoxon signed-rank test, and bootstrap confidence interval.}
\begin{tabular}{|l|ccc|}
\hline
\textbf{Test} & \cellcolor{gray!20}\textbf{Statistic} & \cellcolor{gray!20}\textbf{$p$-value} & \cellcolor{gray!20}\textbf{Outcome} \\
\hline
Paired $t$-test 
& 8.1 
& $4.8{\times}10^{-16}$ 
& Sig diff. \\
Wilcoxon 
& $1.0{\times}10^{7}$ 
& $2.8{\times}10^{-23}$ 
& Sig diff. \\
Bootstrap 
& -- 
& -- 
& 95\% CI $[0.005, 0.006]$ \\
\hline
\end{tabular}
\label{tab:t-test}
\end{table}

\section{More Qualitative Results}

Additional qualitative examples for both datasets are provided in Figures~\ref{fig:all_sevir}–\ref{fig:all_meteo}.
Relative to the recent AlphaPre~\cite{Lin_2025_CVPR} and DuoCast~\cite{wen2026duocast} baseline, HARECast produces sharper forecasts and more accurately tracks medium-intensity precipitation, especially at longer lead times.
Compared with DiffCast~\cite{yu2024diffcast}, both methods perform similarly in the early frames, capturing overall patterns and fine-scale structure.
However, DiffCast’s forecasts deteriorate in later frames, showing spatial distortions and reduced structural coherence, whereas HARECast maintains both detail and stability over extended horizons.

\begin{figure*}[htbp]
  \centering
   \includegraphics[width=0.85\linewidth]{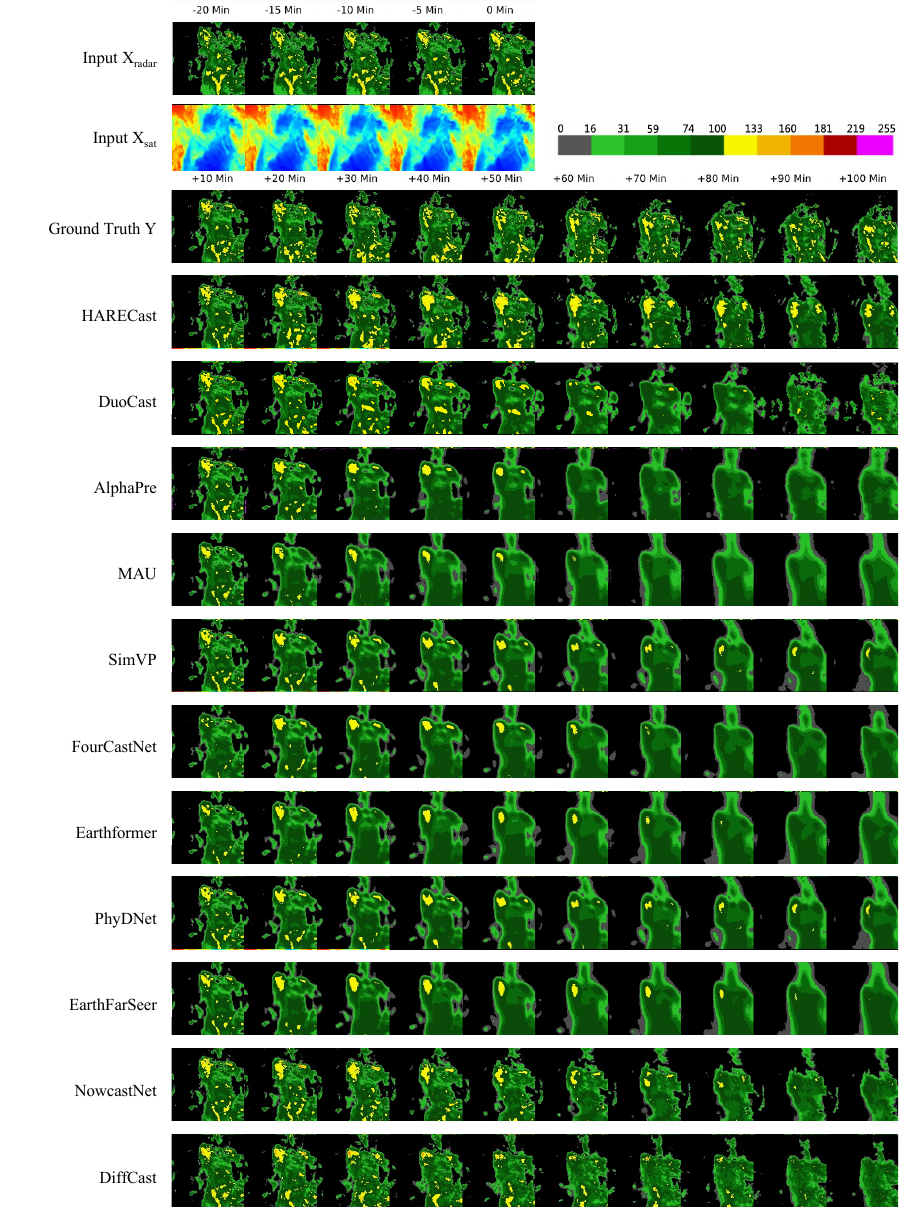}
   \caption{Prediction examples on the SEVIR dataset.}
   \label{fig:all_sevir}
\end{figure*}

\begin{figure*}[htbp]
  \centering
   \includegraphics[width=0.85\linewidth]{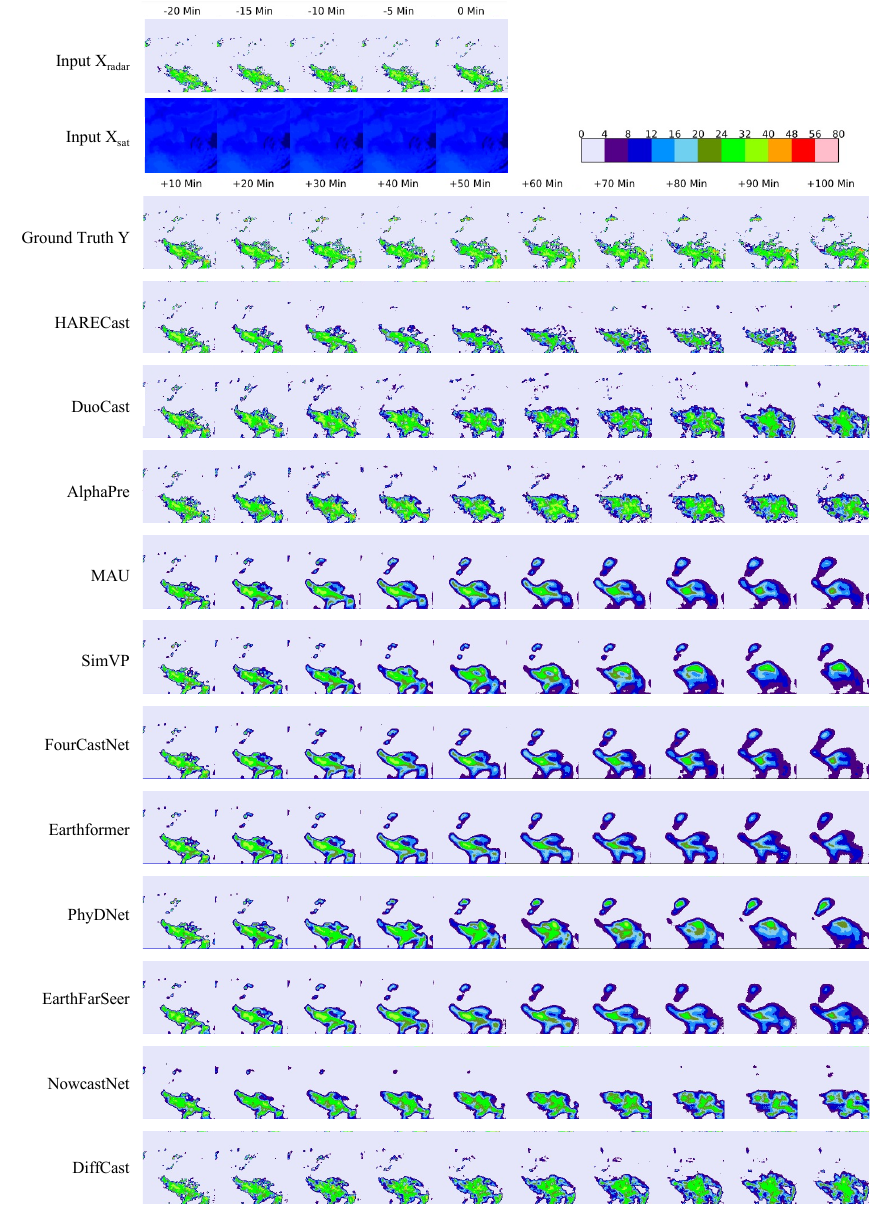}
   \caption{Prediction examples on the MeteoNet dataset.}
   \label{fig:all_meteo}
\end{figure*}

\end{document}